\documentclass{article}

\usepackage{microtype}
\usepackage{graphicx}
\usepackage{subfigure}
\usepackage{booktabs} 
\usepackage{IEEEtrantools}
\usepackage[dvipsnames]{xcolor} 

\usepackage{hyperref}

\usepackage[accepted]{icml2024}

\usepackage{amsmath}
\usepackage{amssymb}
\usepackage{mathtools}
\usepackage{amsthm}
\usepackage{bm}

\usepackage[capitalize,noabbrev]{cleveref}

\theoremstyle{plain}
\newtheorem{theorem}{Theorem}[section]
\newtheorem{proposition}[theorem]{Proposition}

\theoremstyle{definition}

\theoremstyle{remark}

\def\x{{\mathbf x}}
\def\u{{\mathbf u}}
\def\y{{\mathbf y}}
\def\f{{\mathbf f}}
\def\w{{\mathbf w}}
\def\m{{\mathbf m}}
\def\km{{K}}
\def\kv{{\mathbf k}}
\def\R{{\mathbb{R}}}
\def\N{{\mathbb{N}}}
\def\E{{\mathbb{E}}}
\def\X{{\mathcal{X}}}

\def\GP{{\mathcal{GP}}}
\DeclareMathOperator{\Tr}{Tr}
\DeclareMathOperator{\diag}{diag}
\DeclareMathOperator*{\argmax}{arg\,max}

\renewcommand{\epsilon}{\varepsilon}
\def\vepsilon{{\bm{\epsilon}}}

\definecolor{blue_plot}{RGB}{55,  126, 184}
\definecolor{orange_plot}{RGB}{255, 127, 0}
\definecolor{green_plot}{RGB}{77,  175, 74}
\definecolor{pink_plot}{RGB}{247, 129, 191}

\usepackage[textsize=tiny]{todonotes}

\icmltitlerunning{Robust and Conjugate Gaussian Process Regression}

\begin{document}

\twocolumn[
\icmltitle{Robust and Conjugate Gaussian Process Regression}



\icmlsetsymbol{equal}{*}

\begin{icmlauthorlist}
\icmlauthor{Matias Altamirano}{ucl}
\icmlauthor{François-Xavier Briol}{ucl}
\icmlauthor{Jeremias Knoblauch}{ucl}
\end{icmlauthorlist}

\icmlaffiliation{ucl}{Department of Statistical Science, University College London, London, United Kingdom}

\icmlcorrespondingauthor{Matias Altamirano}{matias.altamirano.22@ucl.ac.uk}
\icmlcorrespondingauthor{François-Xavier Briol}{f.briol@ucl.ac.uk}
\icmlcorrespondingauthor{Jeremias Knoblauch}{j.knoblauch@ucl.ac.uk}

\icmlkeywords{Gaussian processes, generalised Bayes, score-matching, robust Bayes}

\vskip 0.3in
]



\printAffiliationsAndNotice{\icmlEqualContribution} 

\begin{abstract}
To enable closed form conditioning, a common assumption in Gaussian process (GP) regression is independent and identically distributed  Gaussian observation noise. 
This strong and simplistic assumption is often violated in practice, which leads to unreliable inferences and uncertainty quantification. 
Unfortunately, existing methods for robustifying GPs break closed-form conditioning, which makes them less attractive to practitioners and significantly more computationally expensive. 
In this paper, we demonstrate how to perform provably robust and conjugate Gaussian process (RCGP) regression at virtually no additional cost using generalised Bayesian inference. 
RCGP is particularly versatile as it enables exact conjugate closed form updates in all settings where standard GPs admit them.
To demonstrate its strong empirical performance, we deploy RCGP for problems ranging from Bayesian optimisation to sparse variational Gaussian processes.
\end{abstract}
\section{Introduction}

GPs \citep{williams2006gaussian} are one of the most widely used methods for Bayesian inference on latent functions, especially when uncertainty is required. 
They have numerous appealing properties, including that the prior is relatively interpretable and can be elicited through a choice of mean and covariance functions, as well as the fact that they have closed form posteriors under Gaussian likelihoods. 
Their convergence is also well understood, even under prior misspecification  \citep{Wynne2020}.
Thanks to these advantages, GPs have found applications in diverse problems including single- and multi-output regression \citep{bonilla2007multi, moreno2018heterogeneous}, emulation of expensive simulators \citep{Santner2018}, Bayesian optimisation \citep{shahriari2015taking,Garnett2021} and Bayesian deep learning \citep{damianou2013deep, salimbeni2019deep, dutordoir2020bayesian}. Their use is enabled by a plethora of packages including \texttt{GPflow} \citep{GPflow2017} \texttt{GPyTorch} \citep{Gardner2018},  \texttt{BoTorch} \citep{balandat2020botorch}, \texttt{ProbNum} \citep{Wenger2021} and \texttt{emukit} \citep{emukit2023}.

 \begin{figure}[t!]
\centering
\includegraphics[width=\columnwidth]{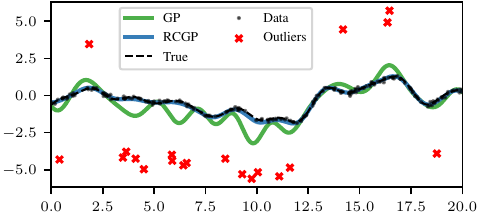}
 \vspace*{-0.7cm}
\caption{The posterior predictive mean of a GP (\textcolor{green_plot}{\textbf{green}}) and the RCGP (\textcolor{blue_plot}{\textbf{blue}}) on a synthetic dataset where 10\% of the data are uniformly generated outliers. 
Unlike the RCGP, the GP is adversely affected.
}

\label{fig:synthetic}
\vspace{-3mm}
\end{figure}

By far the most common use of GPs is in regression. 
Here, the observations correspond to noisy realisations from an unknown
latent function that is assumed to be drawn from a GP prior. 
To obtain a conjugate GP posterior distribution on the latent function, the observation noise is usually assumed to be Gaussian.
While assuming Gaussian observation noise makes the posterior tractable, it also makes inferences non-robust. 
In particular, Gaussian noise makes GPs highly susceptible to extreme values, heterogeneities, and outliers.
 This is illustrated in \cref{fig:synthetic} on a synthetic dataset corrupted with outliers: The standard GP is adversely affected, leading to considerable deviations between the inferred function and the ground truth.
In many real-world applications and data sets, the presence of outliers is almost inevitable.
They can occur for a variety of different reasons, including due to faulty measurements, broken sensors, extreme weather events, stock market sell-offs, or genetic mutations.

\paragraph{Existing Work}
The lack of robustness in GPs is a well-known fundamental challenge for their widespread application, and a number of methods have been proposed to address this.
 Broadly, these fall into two categories.
 The first replaces the Gaussian measurement error with more heavy-tailed error distributions such as Student's $t$  \citep{jylanki2011robust,ranjan2016robust}, Laplace  \citep{kuss2006gaussian}, Huber densities \citep{algikar2023robust}, data-dependent noise \citep{goldberg1997regression}, or mixture distributions  \citep{naish2007robust,stegle2008gaussian,daemi2019identification,Lu2023}. 
 Heavy tails allow these distributions to better accommodate outliers, rendering them more robust to  corruptions.
Their main limitation lies in their computational cost, as
abandoning Gaussian noise nullifies one key advantage of GPs: conjugacy. 
As a consequence, these techniques rely on approximations via variational methods or Markov chain Monte Carlo.
This decreases their accuracy while increasing computational costs.  
The second set of approaches consists in  removing outlying observations before using a standard GP with Gaussian noise \citep{li2021robust,park2022robust,andrade2023robust}.
While such approaches use conjugacy, it can be challenging to detect outliers in irregularly spaced data or higher dimensions. Outlier detection also tends to be computationally costly, and often requires estimating large numbers of parameters.

In this paper, we propose a new and third way to achieve robustness that uses generalised Bayesian inference \citep[see e.g.][]{bissiri2016general, Jewson2018,knoblauch2019generalized}.
In doing so, we significantly improve upon an earlier attempt in this direction due to  \citet{knoblauch2019robust} that was applicable only for variational deep GPs, lacked closed form solutions, and was based on hyperparameters that were difficult to choose. 
In line with the ideas of generalised Bayesian methods, we will not modify the Gaussian noise model. 
Instead, we change how information is assimilated, and leverage robust loss functions instead of  robust error models.

\paragraph{Contributions} This paper proposes a novel robust and conjugate Gaussian process (RCGP) inspired by a generalised Bayesian inference scheme proposed in \citet{altamirano23robust}. 
The posteriors rely on a generalised form of score matching \citep{Hyvarinen2006, barp2019minimum}, which effectively down-weights outlying observations. 
The resulting inference resolves the trade-off between robustness and computation inherent in existing methods:  it is robust in the sense of \citet{huber2011robust} [\cref{prop:robustness}] while retaining closed form solutions for both its posterior and posterior predictive [\cref{prop:RCGP_derivation}].
Additionally---and unlike other robust GPs---RCGPs can easily be plugged into various GP techniques such as sparse variational GPs \citep{titsias2009variational,hensman2013gaussian} [\cref{prop:ELBO}], deep GPs \citep{damianou2013deep}, multi-output GPs \citep{bonilla2007multi}, and Bayesian optimisation \citep{shahriari2015taking} [\cref{prop:acquisition_functions}]. 
Finally, even in settings where robustness is not required, our experiments show that  RCGPs performs as well as standard GPs---raising the possibility that RCGPs may  become a  preferred default choice over GPs in the future.

The remainder of the paper reviews GPs and generalised Bayesian inference (\cref{sec:background}), introduces RCGPs and proves their robustness (\cref{sec:methodology}), and investigates their empirical performance and versatility for a range of experiments  (\cref{sec:experiments}).

\section{Background}
\label{sec:background}
Our method applies the logic of generalised Bayesian posteriors to GPs. Here, we briefly explain the concepts relevant to understanding this interface.

\paragraph{Gaussian Processes}
\label{sec:GP}
Let $\y=(y_1,\ldots,y_n)^\top$ denote $n$  observations with  covariates $\x=(x_1,\ldots,x_n)^\top$, where $y_i\in\mathcal{Y} \subseteq \R$ and $x_i\in\mathcal{X} \subseteq \mathbb{R}^d$. 
While we take $\mathcal{Y}\subseteq \R$ for simplicity, our method can be straightforwardly generalised to  multi-output regression (i.e., $\mathcal{Y} \subseteq \mathbb{R}^T$). 
We consider a regression setting where the noisy observations $\y$ come from a latent function $f:\mathcal{X} \rightarrow \mathbb{R}$: 
\vspace{-2mm}
\begin{IEEEeqnarray}{rCl}
y_i = f(x_i) + \epsilon_i.\nonumber
\end{IEEEeqnarray}
Here, $\vepsilon=(\epsilon_1, \ldots,\epsilon_n)^\top \in \mathbb{R}^n$ are independent observation errors. 
We place a GP prior on $f$, so that $f\sim \GP(m,k)$ with $m:\X \rightarrow \mathbb{R}$ and $k:\X \times \X \rightarrow \mathbb{R}$ being mean and kernel functions.  These functions determine key properties in the draws from the GP such as differentiability, periodicity, long-range correlation or stationarity, and are parameterised by $\theta \in \Theta \subseteq \mathbb{R}^p$.
Throughout, we write 
$\f=(f(x_1),\ldots,f(x_n))^{\top}$, and use the fact that the GP prior implies the Gaussian prior
$p(\f|\x)=\mathcal{N}(\f; \m,\km)$, where $\km$ is the matrix with $K_{ij} = k(x_i,x_j)$  and $\m = (m(x_1),\ldots,m(x_n))^{\top}$.

Finally, while there are various options for modelling observation error, almost all of them break conjugacy. 
The main exception is the choice $\vepsilon \sim \mathcal{N}(0, \sigma^2 I_{n})$ where $I_n$ is an $n \times n$ identity matrix (or equivalently $p(\y|\f,\x)=\mathcal{N}(\y; \f, \sigma^2 I_{n})$).
This leads to the posterior
\begin{IEEEeqnarray}{rCl}
p(\f|\y,\x) &=& \mathcal{N}(\f; \mu, \Sigma),\nonumber\\
\mu &=& \m +\km (\km+\sigma^{2} \mathcolor{green_plot}{I_{n}})^{-1}(\y-\m)  \nonumber\nonumber,\\
\Sigma &=&  \km (\km+ \sigma^{2} \mathcolor{green_plot}{I_{n}})^{-1} \sigma^{2} \mathcolor{green_plot}{I_{n}} \nonumber.
\end{IEEEeqnarray}
A key quantity of interest is then the posterior predictive over $f_\star$, the value $f(x_\star)$ at a new point $x_\star \in \mathcal{X}$:
\vspace{-2mm}
\begin{IEEEeqnarray}{rCl}    p(f_{\star}|x_{\star},\x,\y) & = & \int_{\mathbb{R}^n} p(f_{\star}|x_{\star},\f,\x,\y)p(\f|\y,\x)d\f \nonumber\\
    &=&\mathcal{N}(f_{\star};\mu_{\star},\Sigma_{\star}),\nonumber\\ 
    \mu_{\star}
    &=&m_\star + \kv_{\star}^{\top}(\km+\sigma^{2}\mathcolor{green_plot}{I_{n}})^{-1}(\y-\mathcolor{green_plot}{\m}),\nonumber\\
    \Sigma_{\star}
    &=& k_{\star\star} - \kv_{\star}^{\top}(\km+\sigma^{2} \mathcolor{green_plot}{I_{n}})^{-1}\kv_{\star},\nonumber
\end{IEEEeqnarray}
for  $\kv_{\star} = (k(x_\star, x_1),\ldots,k(x_\star, x_n))^{\top}$, $k_{\star\star} = k(x_\star,x_\star)$, $m_\star=m(x_\star)$. 
Though the posterior also depends on  $\theta$ and $\sigma^2$, we omit this for brevity.
Commonly, they are chosen to maximise the marginal likelihood
\begin{IEEEeqnarray}{rCl}
    p(\y| \x,\theta,\sigma^2) &=&  \int_{\mathbb{R}^n} p(\f | \x, \theta,\sigma^2) p(\y| \f, \x, \theta,\sigma^2) d\f\nonumber \\
    &=&  \mathcal{N}(\y ; \m,\km+\sigma^{2}I_{n}).\nonumber
\end{IEEEeqnarray}
Alternatively, there are computationally efficient approaches for leave-one-out cross-validation, and computationally demanding Markov chain Monte Carlo methods 
for hierarchical Bayes  
\citep[see][Section 5]{williams2006gaussian}.

If $\f$ and $\vepsilon$ are both modelled as Gaussians, GP regression is conjugate, so that
posterior, posterior predictive, and marginal likelihood can all be obtained in closed forms. 
 These operations have $O(n^3)$ computational and $O(n^2)$ storage cost, but are exact.
 For this reason, practitioners often model  data using Gaussian errors---even 
 when this assumption is wholly inappropriate and yields severe misspecification. 
 Note that conjugacy also holds for GP interpolation---when $\vepsilon$ is a Dirac measure at $(0,\ldots,0)^\top$---in which case all formulae above remain correct for $\sigma=0$. 
 However, this assumes that the observations $\y$ are noise-free, which is  even more susceptible to model misspecification than Gaussianity.

\paragraph{Generalised Bayesian Inference}
\label{sec:GBI}

If a statistical model is misspecified so that the model cannot correctly describe the true data-generating mechanism, standard Bayesian updating is not the optimal way of integrating prior information with data \citep{zellner1988optimal}.
Indeed, standard Bayes results in miscalibrated uncertainties and misleading inferences.
In the parametric setting,  a line of research has tackled this through generalised Bayesian methodology \citep[see e.g.][]{grunwald2012safe, Hooker2014,  bissiri2016general, Ghosh2016, Jewson2018, miller2018robust, knoblauch2019generalized, Miller2019, fong2021martingale,Jewson2021, matsubara2021robust}.
Recently, generalised Bayesian posteriors have also been proposed for the non-parametric case \citep[see e.g.][]{knoblauch2019robust, wild2022generalized}. 
For regression, they take the form
\begin{IEEEeqnarray}{rCl}
    p_{\beta}^{L}(\f | \y,\x)\propto p(\f|\x) \exp\big(-\beta n L_n(\f,\y,\x)\big),
    \label{eq:gen-bayes}
\end{IEEEeqnarray}
where $\propto$ denotes equality up to a multiplicative constant not depending on $\f$. 
The \emph{learning rate} $\beta>0$ is a scaling parameter that determines how quickly the posterior learns from data, and $L_n:\mathcal{Y}^n \times \mathcal{Y}^n \times \mathcal{X}^n \rightarrow \mathbb{R}$ is a loss connecting the data and the posited statistical model. 
Most choices for $L_n$ are estimators of statistical divergences between the true data-generating process, $p_0$, and the model, such as the kernel Stein discrepancy \citep{matsubara2021robust}, $\beta$-divergences \citep{knoblauch2018doubly}, maximum mean discrepancies \citep{cherief2020mmd}, and Fisher divergences \citep{altamirano23robust,matsubara2022generalised}. 

The distributions in \eqref{eq:gen-bayes} are called \textit{generalised} posteriors since they recover the standard Bayes posterior for $\beta=1$ and $L_n(\f,\y,\x) = -\log p(\y|\f,\x)$.
By instead choosing $L_n$ to be a robust loss, generalised Bayesian inference has enhanced  applications including filtering \citep{boustati2020generalised,Duran-Martin2024}, changepoint detection  \citep{knoblauch2018doubly, altamirano23robust}, deep Gaussian processes \citep{knoblauch2019robust}, doubly-intractable problems \citep{matsubara2022generalised, matsubara2021robust} and Bayesian neural networks \citep{futami2018variational}. 
In addition, generalised posteriors have been leveraged for computational efficiency. 
For instance, \citet{matsubara2021robust,matsubara2022generalised} used generalised posteriors for accelerated computation with unnormalised models in both continuous and discrete domains. 
Similarly, \citet{schmon2020generalized}, \citet{Dellaporta2022}, \citet{Pacchiardi2021},  \citet{Legramanti2022}, and \citet{frazier2024impact} deployed them for simulation-based inference.
%

\section{Methodology}
\label{sec:methodology}
We now present RCGPs in three steps.
First we introduce our loss and explain how it ensures conjugacy.
Second, we provide formal robustness guarantees.
Finally, we show how to select hyperparameters.

\paragraph{The Loss Function}
\citet{altamirano23robust} present a posterior based on a generalised score matching loss \citep{Hyvarinen2006,Lyu2009,Yu2022} due to \citet{barp2019minimum}. 
The resulting posterior is provably robust, and conjugate for exponential family models.
Importantly however, it is not applicable for regression settings and covariate-dependent models. 
To rectify this, we follow \citet{xu2022generalized} and leverage the tower property of expectations. Using this and denoting by $p_{0,x}$ the marginal distribution of the covariates, score matching losses in our setting lead to 
\begin{IEEEeqnarray}{rCl}
\label{eq:divergence}
    \E_{X\sim p_{0,x}}\left[ 
        \E_{Y \sim p_{0}(\cdot|X)}\left[
        \|(s_{\text{model}} - s_{\text{truth}})(X,Y)\|_{2}^{2}
        \right]
    \right], \nonumber
\end{IEEEeqnarray}
where $s_{\text{truth}}(x,y) = \nabla_{y} \log p_0(y|x)$ is the score function of the true data-generating conditional density, $s_{\text{model}}$ the score function of our model. In the case of GP regression, this is $s_{\text{model}}(x,y) = \sigma^{-2}(f(x)-y)$.
As in \citet{altamirano23robust}, we  instead use a weighted generalisation due to \citet{barp2019minimum}:
\begin{equation}
\label{eq:weighted-fd}
\resizebox{\minof{\width}{.9\hsize}}{!}{$\E_{X\sim p_{0,x}}\left[ \E_{Y \sim p_{0}(\cdot|X)}\left[\|\big(w(s_{\text{model}} - s_{\text{truth}})\big)(X,Y))\|_{2}^{2}\right]\right].$}
\end{equation}
Here, $w:\X\times\mathcal{Y}\to\mathbb{R}\setminus\{0\}$ is a weighting function depending on both $x$ and $y$, and we discuss how it should be chosen at the end of this section. 
Evaluating \eqref{eq:weighted-fd} would require the unknown score $s_{\text{truth}}$.
Luckily, under mild smoothness and boundary conditions \citep{liu2022estimating}, 
 we can use integration by parts to rewrite it
 up to a constant not depending on $f$ as
 \begin{IEEEeqnarray}{rCl}
 \resizebox{\minof{\width}{\hsize}}{!}{%
    $\E_{X\sim p_{0,x}}\left[\E_{Y \sim p_{0}(\cdot|X)}\left[\big((w s_{\text{model}})^{2}  +2\nabla_{y}(w^{2} s_{\text{model}})\big)(X, Y)\right]\right].$ \nonumber}
\end{IEEEeqnarray}
 Importantly, this expression no longer depends on $s_{\text{truth}}$, and only features $p_0$ through an expectation. 
This leads to a natural estimator---and the proposed loss function---which is given by 
\begin{IEEEeqnarray}{rCl}
    L^{w}_n(\f,\y,\x) & = &\dfrac{1}{n}\sum_{i=1}^{n} \big((w s_{\text{model}})^{2}  +2\nabla_{y}(w^{2} s_{\text{model}})\big)(x_i, y_i). \nonumber 
\end{IEEEeqnarray}
While we have motivated the loss using the marginal $p_{0,x}$ for simplicity, $p_{0,x}$ can be replaced with any other measure over $\x$.
The active learning setting is a relevant example for this, and we use $L_n^w$ in a Bayesian optimisation experiment in \cref{sec:experiments} to showcase this.

Finally, when the model is Gaussian, i.e. $p(\y|\f,\x)=\mathcal{N}(\y; \f, \sigma^2 I_{n})$, this loss function becomes quadratic in $f$ as follows:
\begin{IEEEeqnarray}{rCl}
    L^{w}_n(\f,\y,\x) & = & \f^\top A_n \f +b_n^\top \f +c_n \nonumber 
\end{IEEEeqnarray}
for some $A_n,b_n,c_n$ in \cref{app:proof_RCGP_derivation}. This follows the same idea as in  \citet{matsubara2021robust,altamirano23robust}, and will be crucial to obtain closed form solution for both the posterior and predictive.

 \paragraph{Robust and Conjugate Gaussian Processes}
 Based on $L_n^w$, we propose the RCGP posterior
\begin{IEEEeqnarray}{rCl}
    p^{w}(\f | \y, \x) &\propto& p(\f) \exp\{ -n L_n^{w}(\f, \y, \x)\}, \nonumber
\end{IEEEeqnarray}
where we absorb $\beta$ into $w$. 
Considering the same setting as for the standard GP in \cref{sec:GP}, the RCGP posterior and its posterior predictive have closed forms. To state them, take $\diag(\mathbf{v})$ as the $d \times d$ diagonal matrix $D$ so that $D_{ij} = 0$ if $i\neq j$ and $D_{ii} = v_i$. 
\begin{proposition}\label{prop:RCGP_derivation}
Suppose $f \sim \mathcal{GP}(m,k)$ and  $\vepsilon \sim \mathcal{N}(0, I_n \sigma^2)$. Then, the RCGP posterior is 
\begin{IEEEeqnarray}{rCl}
    & & p^{w}(\f | \y, \x)  = \mathcal{N}(\f;\mu^R,\Sigma^R), \nonumber \\
\mu^R &=& \m+ \km\left(\km+\sigma^{2}\mathcolor{blue_plot}{J_{\w}}\right)^{-1}\left(\y-\mathcolor{blue_plot}{\m_\w}\right)  \nonumber\nonumber,\\
\Sigma^R &=&  \km\left(\km+\sigma^{2}\mathcolor{blue_plot}{J_{\w}}\right)^{-1}\sigma^{2}\mathcolor{blue_plot}{J_{\w}} \nonumber,
\end{IEEEeqnarray}
for $\w = (w(x_1,y_1),\ldots,w(x_n,y_n))^\top$, $\m_\w = \m +\sigma^{2}\nabla_{y}\log(\w^{2})$ and $J_{\w} =\diag(\frac{\sigma^{2}}{2}\w^{-2})$. 
The RCGP's posterior predictive over $f_{\star} = f(x_{\star})$ at $x^{\star} \in \mathcal{X}$ is 
\begin{IEEEeqnarray}{rCl}
    p^{w}(f_{\star}|x_{\star},\x,\y) & = & \int_{\mathbb{R}^n} p(f_{\star}|x_{\star},\f,\x,\y)p^{w}(\f|\y,\x)d\f \nonumber\\
    &=&\mathcal{N}(f_{\star};\mu_{\star}^{R},\Sigma_{\star}^{R}),\nonumber
\end{IEEEeqnarray}
\vspace{-8mm}
\begin{IEEEeqnarray}{rCl}
    \mu_{\star}^{R}
    &=&m_\star + \kv_{\star}^{\top}\left(\km+\sigma^{2}\mathcolor{blue_plot}{J_{\w}}\right)^{-1}\left(\y-\mathcolor{blue_plot}{\m_\w}\right),\nonumber\\
    \Sigma_{\star}^{R}
    &=& k_{\star\star}-\kv_{\star}^{\top}(\km+\sigma^{2}\mathcolor{blue_plot}{J_{\w}})^{-1}\kv_{\star}.\nonumber
\end{IEEEeqnarray}
\end{proposition}
Throughout, exponents are applied entry-wise, and the proofs can be found   in \Cref{app:proof_RCGP_derivation}. 
The distributions derived in the result have the same structure as their  standard GP counterparts, but replace $\sigma^2 I_n$ with the ``noise term'' $\sigma^2 J_{\w}= \sigma^2 \diag(\frac{\sigma^{2}}{2}\w^{-2})$, and $\m$ with the ``shrinkage term''  $\m_\w = \m+\sigma^{2}\nabla_{y}\log(\w^{2})$.
We interpret both terms after discussing how $w$ should be chosen.

\begin{table}[t]
\caption{Existing methods as special cases of RCGPs}
\centering
{
\label{tab:w}
\begin{tabular}{@{}lc@{}}
\toprule
Method             & $w(x,y)$                \\ \midrule
Standard GP        & $\frac{\sigma}{\sqrt{2}}$                                  \\
Heteroskedastic GP \hspace*{1.2cm}& $\frac{\sigma^2}{\sqrt{2}} \cdot r(x)^{-1}$ \\
Robust GP & $\beta\cdot \big(1+\frac{(y-m(x))^2}{c^2}\big)^{-1/2}$\\ \bottomrule
\end{tabular}
}
\vspace{-2mm}
\end{table}

Similar to standard GP regression, RCGP regression is conjugate, has a computational cost of $O(n^3)$, and storage cost of $O(n^2)$. In addition, a variety of GP schemes fall into the proposed framework through a specific choice of $w$.
For example, $w(x,y) = \frac{\sigma}{\sqrt{2}}$ recovers the standard GP, and $w(x,y) = \frac{\sigma^2}{\sqrt{2}} \cdot r(x)^{-1}$ a heteroskedastic GP with noise rate $r(x)$; see  \cref{tab:w}. Importantly, it would be flawed to simply interpret RCGPs as GPs with a different noise model: $\w$ depends directly on  $\y$.

Finally, we note that RCGPs are in principle not suited to interpolation: the score $s_{\text{model}}$ and thus the loss $L_n^w$ are not defined if $\sigma^2=0$. 
However, it is common to add a ``nugget'' to the kernel for  regularisation \citep{Andrianakis2012}, which makes the problem equivalent to regression with very small  $\sigma^2>0$. 
Consequently, RCGP regression is still applicable for interpolation problems whenever this regularisation is used---and can in fact strongly improve robustness in this setting.
 
\paragraph{Hyperparameter Selection}
To make RCGPs practically viable, we need to estimate $\theta$ and $\sigma^2$. A first idea would be to 
maximise the pseudo marginal likelihood $p^{w}(\y| \theta,\sigma^2)$---the RCGP's equivalent to the marginal likelihood, whose closed form is given in \Cref{app:proof_RCGP_marginal_likelihood}. 
Unfortunately, this would be ill-posed: neither $\exp\{ -n L^{w}(f, \y, \x)\}$ nor $p^{w}(\y| \theta,\sigma^2)$ are probability densities over $\y$.
 Hence, maximising $p^{w}(\y| \theta,\sigma^2)$ is like maximum likelihood estimation for an un-normalised density whose normaliser depends on the estimated parameter, leading   to implausible hyperparameters and numerical issues---a well-known issue for generalised Bayesian methods \citep{Jewson2021}.
 Instead of $p^{w}(\y| \theta,\sigma^2)$, we thus maximise the leave-one-out cross validation (LOO-CV) predictive posteriors via
\begin{IEEEeqnarray}{rCl}
    \hat{\sigma}^2, \hat{\theta} & = &
    \argmax_{\sigma^2, \theta}\Big\{
    \sum_{i=1}^{n}\log p^{w}(y_i|\x,\y_{-i},\theta,\sigma^2)
    \Big\},
    \nonumber
\end{IEEEeqnarray}
where $\y_{-i}=(y_1,\ldots,y_{i-1},y_{i+1},\ldots,y_{n})$ \citep[for details, see Section 5.2][]{williams2006gaussian}. LOO-CV has been one of the primary methods for setting hyperparameters in Gaussian Processes \citep[see e.g.][]{sundararajan1999predictive,bachoc2013cross,vehtari2016bayesian,petit2020towards}. A naive implementation of this objective function would require fitting the model $n$ times, leading to a computational cost of $\mathcal{O}(n^4)$. Hence, we follow \citet{sundararajan1999predictive} to obtain an analytical formulation that allows us to compute the LOO objective in $\mathcal{O}(n^3)$.
By \cref{prop:RCGP_derivation}, $p^w(y_{i}|\x,\y_{-i},\theta,\sigma^2)=\mathcal{N}(\mu_{i}^{R}, \sigma_{i}^{R} +\sigma^2)$ with
\begin{IEEEeqnarray}{rCl}
    \mu_{i}^{R} &=& z_{i}+\m_{i}-[\left(\km+\sigma^{2}J_{\w}\right)^{-1}\mathbf{z}]_{i}  
    [(K+\sigma^2 J_\w)^{-1}]_{ii}^{-1},
    \nonumber\\
    \sigma_{i}^{R}&=&
    [(K+\sigma^2 J_\w)^{-1}]_{ii}^{-1} -  \frac{\sigma^4}{2} w(x_i, y_i)^{-2}, 
    \nonumber
\end{IEEEeqnarray}
for $\mathbf{z} = \y-\m_\w$ and $\mathbf{z} = (z_1, \dots z_n)$. The full derivation can be found in \cref{app:loo}. Crucially, we only need to compute $(\km+\sigma^{2}J_{\w})^{-1}$ once which leads to a computational cost of $\mathcal{O}(n^3)$, same as the standard marginal likelihood method.

\begin{figure}[t!]
\centering
\includegraphics[width=\columnwidth]{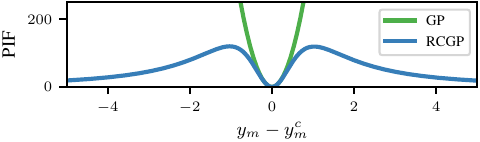}
 \vspace*{-0.7cm}
\caption{
    $\text{PIF}_{\text{GP}}$ (\textcolor{green_plot}{\textbf{green}}) and $\text{PIF}_\text{RCGP}$  (\textcolor{blue_plot}{\textbf{blue}}) for the dataset in \cref{fig:synthetic}. 
    $\text{PIF}_{\text{GP}} \rightarrow \infty $ as $|y_{m}-y^{c}_{m}|\to \infty$, so that standard GPs are not robust. 
    In contrast, $\text{PIF}_{\text{RCGP}}$ is bounded, showing robustness of the RCGP.
}
\vspace{-2mm}
\vspace{-3mm}
\label{fig:pif}
\end{figure}

\paragraph{Robustness}
RCGPs are  not only computationally attractive, but also robust to outliers and non-Gaussian errors.
While robustness for Bayesian methods can refer to a number of other aspects, including calibration \citep[][]{grunwald2012safe, huggins2019robust, Lyddon2019}, adversarial robustness \citep{Bogunovic2020,Kirschner2021}, and robustness to misspecified priors 
\citep{Vaart2011,Bachoc2013,Teckentrup2020,Wang2020,Karvonen2021,bogunovic2021misspecified,Stephenson2022,Naslidnyk2023},
we prove RCGP's robustness to misspecification in the error model.
We do so using the classical framework of \citet{huber2011robust}.
To this end, we first define the contamination of a dataset $D=\{(x_i, y_i)\}_{i=1}^{n}$ by the datum $y_m^c$ as  $D^{c}_{m}=(D \setminus \{(x_m, y_m)\}) \cup \{(x_m,y^{c}_{m})\}$ for some $m\in\{1,\ldots,n\}$.
We quantify the impact of  $y_m^c$ on inference through the divergence between the contaminated and uncontaminated posteriors.
As a function of $|y_m^c - y_m|$, this divergence is also sometimes called the \emph{posterior influence function} (PIF) and was studied for parametric models in \citet{Ghosh2016} and \citet{matsubara2021robust}.
To operationalise this, we consider  the Kullback-Leibler (KL) divergence:
\begin{IEEEeqnarray}{rCl}
    \operatorname{PIF}(y^{c}_{m}, D) = \operatorname{KL} \left(
        p^{L}_{\beta}(\f | D) 
        \| p^{L}_{\beta}(\f | D^{c}_{m})\right). \nonumber
\end{IEEEeqnarray}
We call a posterior robust if  $\sup_{y^{c}_{m}\in\mathcal{Y}}|\operatorname{PIF}(y^{c}_{m}, D)|<\infty$, as this implies  that even as  $|y_m^c-y_m| \to \infty$, the contamination's effect  on the posterior (as measured by the KL) is bounded. 
Choosing  the KL is convenient as it allows  closed form expressions for Gaussians, but we could in principle pick any other divergence with closed form expressions that is not uniformly bounded.
\begin{proposition}
\label{prop:robustness}
Suppose $\f \sim \mathcal{GP}(m,k)$,  $\vepsilon \sim \mathcal{N}(0, I_n \sigma^2)$, and let $C_k \in \R; k=1,2,3$ be constants independent of $y_m^c$. 
Then, GP regression has the PIF
\begin{IEEEeqnarray}{rCl}
       && \operatorname{PIF}_{\operatorname{GP}}(y^{c}_{m}, D) = C_1 (y_m-y^{c}_{m})^{2}, \nonumber
    \end{IEEEeqnarray}
    and is not robust: $\operatorname{PIF}_{\text{GP}}(y^{c}_{m}, f, D) \rightarrow   \infty$ as $|y^{c}_{m}| \to  \infty$.
In contrast, for RCGPs with  $\sup_{x,y} w(x,y) < \infty$, 
    \begin{IEEEeqnarray}{rCl}
            \operatorname{PIF}_{\operatorname{RCGP}}(y^{c}_{m}, D) & \leq &  C_2 (w(x_{m},y^{c}_{m})^2 y^{c}_{m})^2 + C_3.
       \nonumber
    \end{IEEEeqnarray}
    Thus, if $\sup_{x,y}\left\{y \cdot w(x,y)^2\right\} < \infty$,  RCGP is robust since $\sup_{y^{c}_{m}} |\operatorname{PIF}_{\operatorname{RCGP}}(y^{c}_{m},  D)| < \infty$.
\end{proposition}
The proof is in \Cref{app:proof_robustness}.
The two conditions on $w$ in \cref{prop:robustness} have clear interpretations: $\sup_{x,y}w(x,y)<\infty$ ensures that no observation has infinite weight---which would be antithetical to robustness.
Further,  $\sup_{x,y}\left\{y \cdot w(x,y)^2\right\} < \infty$ demonstrates that for  robustness, $w$ must down-weight observations at least at rate $1/y$ for $|y|$  large enough. 
\cref{fig:pif} shows the PIF for an RCGP with a weighting function that satisfies these conditions.
Notably, the  conditions on $w$ in \cref{prop:robustness} can only hold if $w$ depends on $y$.
A consequence is that the $w$ associated with  heteroskedastic GPs in \cref{tab:w} does not lead to a bounded PIF---and is  \textit{not} robust in the sense defined above. 

\label{sec:choice_w}
\begin{figure}[t!]
\centering
\includegraphics[width=\columnwidth]{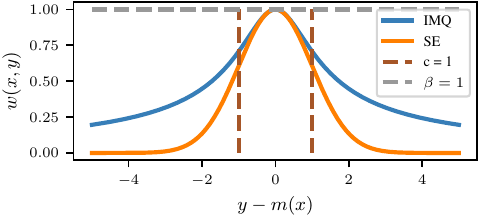}
 \vspace*{-0.7cm}
\caption{
Comparing kernel-based $w$ with the same hyperparameters: IMQ (\textcolor{blue_plot}{\textbf{blue}}) and Squared Exponential (SE) (\textcolor{orange_plot}{\textbf{orange}}).  
The dashed vertical lines indicate the soft threshold $c$ past which a point is increasingly treated as an outlier. 
The SE down-weights observations more rapidly as they exceed $c$ than the IMQ.
The maximum possible weight for any observation is $\beta = 1$.
}
\label{fig:w}
\end{figure}
 \begin{figure*}[h]
     \centering
     \includegraphics{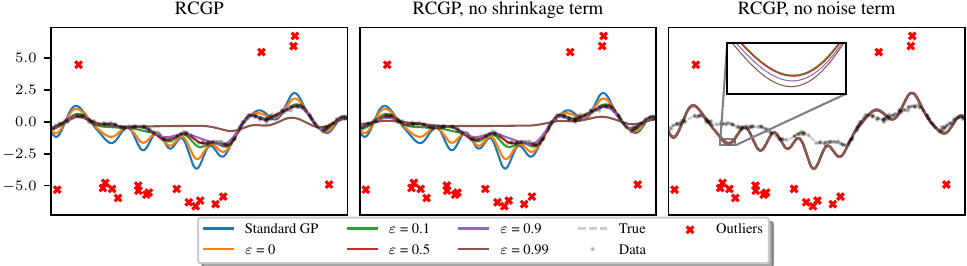}
         \caption{Posterior predictive mean for varying values of $c$ obtained by adjusting $\epsilon$ using the quantile absolute deviation method proposed in \cref{sec:choice_w}, applied to a synthetic dataset with 10\% uniformly generated outliers. \textbf{Left:} Full RCGP. \textbf{Center:} RCGP with no shrinkage term. \textbf{Right:} RCGP with no noise term.}
     \label{fig:varying_c_fixing_w}
 \end{figure*}
\paragraph{Choice of Weighting Function}
Many choices of $w$ satisfy the conditions of \cref{prop:robustness}.
An ideal choice of $w$ attains its maximum close to reasonable data points, and  decreases as outliers become more extreme.
In the remainder, our measure of how outlying an observation $y_i$ is will be $|y_i - m(x_i)|$. 
Additionally, $w$ should be smooth to ensure  its derivatives exist.
Conveniently, this makes weighting functions constructed via infinitely differentiable radial kernels $k$ as $w(x_i,y_i) = k(y_i,m(x_i))$ well-suited. 
We advocate for the Inverse Multi-Quadratic (IMQ) kernel: it has heavy tails, so that  extreme observations are not weighted down too much; see \cref{fig:w}. It is given by
\begin{IEEEeqnarray}{rCl}
    w_{\text{IMQ}}(x,y)&=&\beta\Big(1+\frac{(y-m(x))^2}{c^2}\Big)^{-\frac{1}{2}},\nonumber
\end{IEEEeqnarray}
with $\beta,c>0$. 
Weighting functions chosen this way have two hyperparameters: the soft threshold $c$, and the learning rate $\beta$ from \eqref{eq:gen-bayes} that we pulled into $w$.
While it is in principle possible to pick both $\beta$ and $c$ jointly, we fix $\beta=\frac{\sigma}{\sqrt{2}}$ and choose only $c$.
We do so since joint selection of $\beta$ and $c$ is numerically unstable due to near non-identifiability.
While we could fix $c$ and select $\beta$, this is difficult since learning rate selection is largely unresolved \citep[see e.g.][]{Lyddon2019,Syring2019,Bochkina2022,Wu2023,Frazier2023}.
Fixing $\beta =\frac{\sigma}{\sqrt{2}}$ and choosing $c$ is much easier as $c$ is  interpretable. 
In particular, as $c\to \infty$, $w \to \frac{\sigma}{\sqrt{2}}$ so that the limiting setting recovers the standard GP (cf. \cref{tab:w}). On the other hand,  finite values $c<\infty$ can be interpreted as a soft outlier threshold: Points $y_i$ for which $|y_i - f(x_i)| = c + \xi$ are increasingly treated as outliers the larger $\xi$ becomes.
This is illustrated in \cref{fig:w}, which depicts two kernels with  $c=1 = \beta$, and shows that the weights decrease rapidly to $0$ once the threshold is exceeded.

While, in principle, it is possible to choose $c$ by maximising the leave-one-out cross-validation predictive posterior, the performance is not optimal in practice. This is likely because maximising the predictive posterior for extreme observations tends to match/fit these outliers by \textit{increasing} $c$, leading to a less robust method. See \cref{app:experiment_benchmarking} for further discussion.
Therefore, we propose choosing $c$ via the quantile absolute deviation around the prior mean as $c = {Q}_n(1-\varepsilon)$, where ${Q}_n(1-\varepsilon)$ is the $(1-\varepsilon)$-th quantile of $\{|y_i - m(x_i)|\}_{i=1}^n$ for $\varepsilon\in[0, 1]$. 
As a default setting, we suggest $\varepsilon=0.05$, which implies that we expect at most $5\%$ of the data to be outliers. 

\cref{fig:varying_c_fixing_w} shows an example of the impact of the choice of $\varepsilon$ in the posterior where we observe that choosing $\epsilon=0.99$ will lead inferences that are \textit{too} robust and treat nearly all points as outliers. Conversely, setting $\epsilon=0$ nearly reproduces the non-robust standard GP. Intermediate values exhibit a well-balanced and effective performance.

Crucially, the proposed weighting function depends on the prior mean $m(x_i)$. Hence, if the prior mean is not well specified, the method may discard observations that are not outliers or vice versa. (see \cref{fig:flashcrash}). Therefore, we emphasise the importance of carefully selecting the prior mean, as this choice could impact the performance of our method. Note that this consideration is crucial not only for our approach, but also for the standard GP \citep[see e.g.][]{de2020you,hwang2023use}
\vspace{-2mm}
\paragraph{Interpretation of RCGP Terms}
Having discussed our choice of $w$, we are now ready to interpret the new terms arising in \cref{prop:RCGP_derivation}. 
For the RCGP, the noise term $\sigma^2 J_\w = \sigma^2 I_n (\mathbf{1} + (\y - \m)^2/c^2)$ replaces the standard GP term $\sigma^2 I_n$.  Here, the exponents have been applied element-wise and $\mathbf{1} = (1,\dots 1)^\top \in \R^n$.
Functionally,  $\sigma^2 J_\w$ treats outliers as though they were noisier than other observations. 
Note however that the weight for $y_i$ depends on $y_i$ itself---and thus \textit{cannot} be interpreted as coming from a different or heteroskedastic noise model on $\vepsilon$.
Conceptually, the term treats  information as unreliable whenever it comes from observations $y_i$ that deviate very extremely from the prior mean $m(x_i)$.
This not only ensures robustness, but also has computational benefits: in particular, the term tends to improve  conditioning, making it more numerically stable to invert $K+\sigma^2 J_\w$ than $K+\sigma^2 I_n$. 
This is especially relevant when $\sigma$ is small (e.g. for interpolation) and $n$ large, in which case numerical conditioning is typically a significant issue for GPs \citep{Andrianakis2012}.

Next, $\y - \m_\w  = \y - \m - \sigma^2\left[ \frac{2(\y - \m)}{c^2  \mathbf{1} + (\y - \m)^2} \right]$  replaces the standard GP's $\y - \m$ for RCGPs.
Unlike the noise term, the $i$-th shrinkage term $[\m_{\w}]_i$ is \textit{not} monotonically increasing in $|y_i - m(x_i)|$. 
Rather, it is increasing as $|y_i - m(x_i)| \uparrow c$, peaks at $\sigma^2 / c$ for $|y_i - m(x_i)|=c$, and then decreases monotonically  as $|y_i - m(x_i)|$ increases further.
The effect is most obvious for $\m = 0$, for which $\y - \m_\w = \y(\mathbf{1} - 2\sigma^2[c^2  \mathbf{1} + \y^2]^{-1})$.
Generally, shrinkage is associated with trading off a slight bias for a reduction in variance \citep[e.g.][]{von2011statistical}.
This also applies here: 
relative to the GP, the RCGP's posterior mean will exhibit smaller variance at the expense of a slight bias towards $m$.

\cref{fig:varying_c_fixing_w} shows the impact of each term in the predictive posterior. It is clear that the main term providing robustness to our method is the noise term.

%
\section{Experiments}
\label{sec:experiments}
The code 
is available at \url{https://github.com/maltamiranomontero/RCGP}.
The extensive literature on robust GPs makes it impossible to compare RCGPs to all competitors.
Thus, we choose two representative and popular approaches: a GP with heavy tailed Student's $t$  errors  \citep[``t-GP'']{jylanki2011robust}, and a GP directly modelling outliers via mixture distributions \citep[``m-GP'']{Lu2023}. 
We do not compare against outlier-removal methods: 
they complement---rather than compete with---RCGPs. Throughout, we picked  $w$, $c$ and $\beta$ as proposed in \cref{sec:choice_w}. 
All hyperparameters are selected via L-BFGS.

\begin{table}[t]
\caption{Average test set mean absolute error and standard deviation (in brackets) for 50 train--test splits.}
{\small
\label{tab:benchmark}
\centering
\begin{tabular}{@{}lllll@{}}
\toprule
          & \multicolumn{1}{c}{\textcolor{green_plot}{\textbf{GP}}}       & \textcolor{blue_plot}{\textbf{RCGP}}      & \textcolor{orange}{\textbf{t-GP}}  & \textcolor{purple}{\textbf{m-GP}}    \\ \midrule
          \multicolumn{5}{c}{{{\footnotesize No Outliers}}} 
          \vspace{0.5mm} \\ 
Synthetic & \textbf{0.09 (0.00)}& \textbf{0.09 (0.00)} & \textbf{0.09 (0.00)} & 0.33 (0.00) \\
Boston    & \textbf{0.19 (0.01)}& \textbf{0.19 (0.01)} & \textbf{0.19 (0.01)} & 0.28 (0.00) \\
Energy    & 0.03 (0.00)         & \textbf{0.02 (0.00)} & 0.03 (0.00)          & 0.61 (0.00) \\
Yacht     & 0.02 (0.01)         & 0.02 (0.01)          & \textbf{0.01 (0.00)} & 0.33 (0.00)\\ \midrule
\multicolumn{5}{c}{{{\footnotesize Focused Outliers}}}  \vspace{0.5mm}                          \\
Synthetic & 0.19 (0.00)   & \textbf{0.15 (0.00)} & 0.18 (0.00) & 0.23 (0.00)\\
Boston    & 0.23 (0.06)   & \textbf{0.22 (0.01)} & 0.27 (0.00)  & 0.27 (0.00)\\
Energy    & 0.03 (0.04)   & \textbf{0.02 (0.00)} & 0.03 (0.05) & 0.24 (0.00)\\
Yacht     & 0.26 (0.15)   & \textbf{0.10 (0.14)} & 0.20 (0.04) & 0.24 (0.00)\\ \midrule
\multicolumn{5}{c}{{{\footnotesize Asymmetric Outliers}}} 
\vspace{0.5mm} \\
Synthetic & 1.14 (0.00)  & 0.63 (0.00)          & 1.06 (0.00) & \textbf{0.61 (0.00)} \\
Boston    & 0.63 (0.02)  & \textbf{0.49 (0.00)} & 0.52 (0.00) & 0.52 (0.00) \\
Energy    & 0.54 (0.02)  & 0.44 (0.04)          & 0.42 (0.02) & \textbf{0.41 (0.00)} \\
Yacht     & 0.54 (0.06)  & \textbf{0.35 (0.02)} & 0.41 (0.00) & 0.40 (0.00) \\
 \midrule
          \multicolumn{5}{c}{{{\footnotesize Uniform}}} 
          \vspace{0.5mm} \\ 
Synthetic & 0.34 (0.00)  & \textbf{\textbf{0.21 (0.00)}} & 0.30 (0.00) & 0.27 (0.00) \\
Boston    & 0.53 (0.13)  & 0.51 (0.12) & 0.30 (0.01) & \textbf{0.24 (0.00)} \\
Energy    & 0.25 (0.03)  & 0.24 (0.03) & \textbf{0.23 (0.05)} & 0.23 (0.00) \\
Yacht     & 0.36 (0.16)  & 0.29 (0.07) & \textbf{0.16 (0.00)} & 0.34 (0.00)\\
\bottomrule
\end{tabular}
}
\end{table}

\begin{table}
\caption{
    Average clock time in seconds and its standard deviation (in brackets) across 50 repetitions
}
\label{tab:benchmark_time}
\centering
{\small
\begin{tabular}{@{}lllll@{}}
\multicolumn{5}{c}{}                           \\ \midrule
          & \textcolor{green_plot}{\textbf{GP}}   & \textcolor{blue_plot}{\textbf{RCGP}} & \textcolor{orange}{\textbf{t-GP}} & \textcolor{purple}{\textbf{m-GP}}      \\ \midrule
Synthetic \hspace*{0.8cm} & 1.5 (0.1)    & 1.2 (0.0) & 2.2 (0.0) & 3.0 (0.0)\\
Boston    & 1.9 (0.5)    & 5.1 (0.9) & 30.7 (6.1) & 16.7 (1.7)\\
Energy    & 3.8 (0.9)    & 4.6 (2.0) & 34.0 (11) & 33.8 (0.3)\\
Yacht     & 1.6 (0.3)    & 2.1 (0.2) & 5.6 (0.7) & 4.5 (0.4)\\ \bottomrule
\end{tabular}
}
\end{table}

\paragraph{Benchmarking}

\begin{figure}[t]
    \centering
    \includegraphics{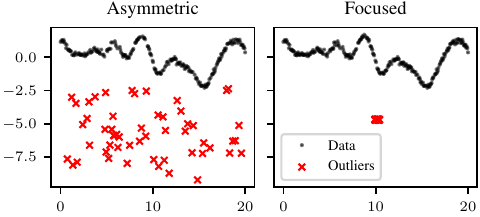}
    \vspace*{-0.25cm}
    \caption{Considered contamination regimes are
    asymmetric (left) and focused (right).
    }
    \vspace{-2mm}
    \label{fig:outliers}
\end{figure}

We assess our method on four benchmark datasets, including the synthetic problem in  \Cref{fig:synthetic} ($d=1, n=300$), the \href{https://www.cs.toronto.edu/~delve/data/boston/bostonDetail.html}{\texttt{Boston} dataset} ($d=13, n=506$), and two datasets from the \href{https://archive.ics.uci.edu}{UCI repository}: \texttt{Energy} ($d=8, n=768$) and \texttt{Yacht} ($d=6, n=308$). These allow us to compare the performances of robust GP methods in a range of dimensions and number of data points. 
For each dataset, we consider four settings: 
the original dataset (i.e. no outliers), focused outliers (i.e. outliers clustered in x- and y- space), asymmetric outliers (i.e. observations corrupted by strictly negative shifts) and uniform outliers (both positive and negative shifts at random); see Figures \ref{fig:synthetic} and \ref{fig:outliers} for illustrations. In each case, $10\%$ of observations are perturbed to become outliers.
Results are provided in Tables \ref{tab:benchmark} and \ref{tab:benchmark_time}, and a full description of the datasets and outlier generation process is provided in \Cref{app:experiment_benchmarking}.

As expected, standard GPs  outperform robust methods in absolute mean error when there are no outliers.
That being said, RCGPs can easily compete in this setting---indicating that there are no clear drawbacks to using RCGPs as a default. 
For example, on both the \texttt{Energy} and \texttt{Boston} datasets, the performance of RCGPs in the setting with no outliers is comparable to that of GPs. 
Further, RCGPs  tend to outperform their competitors in the presence of focused and asymmetric outliers. 
This happens as the competitors' posited noise models are symmetric, but the outlier generation is not.
This phenomenon demonstrates the advantage of using a generalised Bayesian approach as opposed to a more refined noised model. 
In the uniform outlier setting 
, t-GPs outperform all alternatives.
This is to be expected: Student's $t$ errors are a reasonable approximation to that error generating process. 
Though t-GPs perform even better, RCGPs still significantly outperform GPs. 
Further, while RCGPs and GPs have comparable computational cost,  t-GPs and m-GPs are significantly more computationally expensive.
To demonstrate this,  \Cref{tab:benchmark_time} compares the cost of training (including hyperparameter selection) of each method.
The slowdown of t-GPs and m-GPs is due to needing  variational approximations, and especially noticeable for datasets with larger $n$. 
Both GPs and RCGPs have runtimes of $O(n^3)$, though there are minor numerical differences due to the adaptive stopping rule of L-BFGS for hyperparameter optimisation. 

\paragraph{Sparse Variational Gaussian Processes (SVGPs)} The $O(n^3)$ cost of GPs is prohibitive for large $n$. 
A popular remedy are SVGPs \citep{titsias2009variational}, which reduce the cost to $O(nm^2)$ with $m \ll n$ inducing points $\u = (u_1, \dots, u_m)^\top$. 
RCGPs are amenable to this type of inference.
The resulting robust conjugate SVGP (RCSVGP) is derived below.

\begin{figure}[t]
    \centering
    \includegraphics{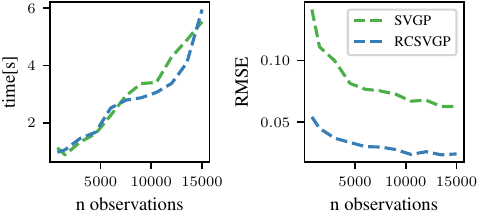}
     \vspace*{-0.25cm}
    \caption{
    Average clock time (left) and root mean square error (right) for
    SVGP (\textcolor{green_plot}{green}) and RCSVGP (\textcolor{blue_plot}{blue}) on the synthetic dataset with 10\% uniform outliers and $m=\sqrt{n}$ inducing points across 10 repetitions.
    }
    \label{fig:sparse}
    \vspace{-2mm}
\end{figure}

\begin{proposition}\label{prop:ELBO}
For $f \sim \mathcal{GP}(m,k)$, $\vepsilon \sim \mathcal{N}(0,\sigma^2)$, the RCSVGP posterior is $f \sim \mathcal{GP}(\widetilde{\mu}, \widetilde{\Sigma})$, where
\begin{IEEEeqnarray}{rCl}
   \widetilde{\mu}(x) 
   & = &
    \phi_{\u}(x)^{\top} 
    \mu_{\u}, \nonumber \\
    \widetilde{\Sigma}(x, x')
    & = &
    k(x, x') - 
    \phi_{\u}(x)^{\top}
    \left( K_{\u\u} - \Sigma_{\u} \right) \phi_u(x'),
    \nonumber \\
    \mu_{\u} & = &
    \m + K_{\u\u} P_{\u}^{-1} K_{\u}\sigma^{-2}J_{\w}^{-1}(\y-m_{\w}),
    \nonumber \\
    \Sigma_{\u} & = &
     K_{\u\u} P_{\u}^{-1} K_{\u\u},
    \nonumber
\end{IEEEeqnarray}
for  $P_{\u} = \left(K_{\u\u}+K_{\u}^{\top}\sigma^{-2}J_{\w}^{-1}K_{\u}\right)$, $[K_{\u\u}]_{ij} = k(u_i, u_j)$, $[K_{\u}]_{ij} = k(u_i, x_j)$, $[k_{\u}(x)]_i = k(u_i, x)$, and $\phi_{\u}(x) = K_{\u\u}^{-1}k_{\u}(x)$. 
\begin{IEEEeqnarray*}{rCl}
 J(\u,\theta,&& \sigma^2)=   \dfrac{1}{2}\nu^{\top}K_{\u}^{\top}Q_{\u}^{-1}K_{\u}\nu  + \dfrac{1}{2}\log{\left(\frac{
 \det\left(K_{\u\u}\right)^2}{\det\left(Q_{\u}\right)}\right)} \\&& + C(\sigma^2) - \Tr{\Big(\sigma^{-2}J_{\w}^{-\frac{1}{2}}(K-K_{\u}^{\top}K_{\u\u}^{-1}K_{\u})J_{\w}^{-\frac{1}{2}}\Big)}, 
\end{IEEEeqnarray*}
where
$Q_{\u} = K_{\u\u} + K_{\u}^{\top}\sigma^{-2}J_{\w}^{-1}K_{\u}$, $\nu = \sigma^{-2} J_{\w}^{-1}(\y-\m_{\w})$, and
$C(\sigma^2)$ is a function that only depends on $\sigma^2$.
\end{proposition}

See \Cref{app:proof_ELBO} for a proof and  details. 
A numerical comparison between SVGPs and RCSVGPs for the synthetic example with 10\% uniform outliers is presented in \Cref{fig:sparse}. 
While RCSVGPs and SVGPs have similar runtimes,  RCSVGPs make far more accurate predictions. 
In \Cref{app:SVGP}, we show that this generalises to the other datasets and outlier types.

\begin{figure}[t]
    \centering
    \includegraphics{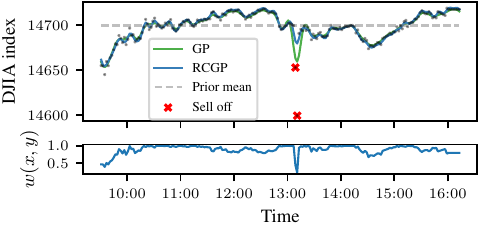}
     \vspace*{-0.25cm}
    \caption{\textit{Top:} GP (\textcolor{green_plot}{green}) and RCGP (\textcolor{blue_plot}{blue}) regression on the DJIA with $m(x) = \frac{1}{n}\sum_{i=1}^n y_i$. 
    While the GP overfits to outliers around 1pm,  RCGP is robust. 
    \textit{Bottom:}  $w_{\text{IMQ}}$ with $\beta$, $c$ chosen as proposed in \cref{sec:choice_w}.
    }
    \label{fig:flashcrash}
    \vspace{-3mm}
\end{figure}
\paragraph{Twitter Flash Crash} 
\begin{figure*}[t]
    \centering
    \includegraphics{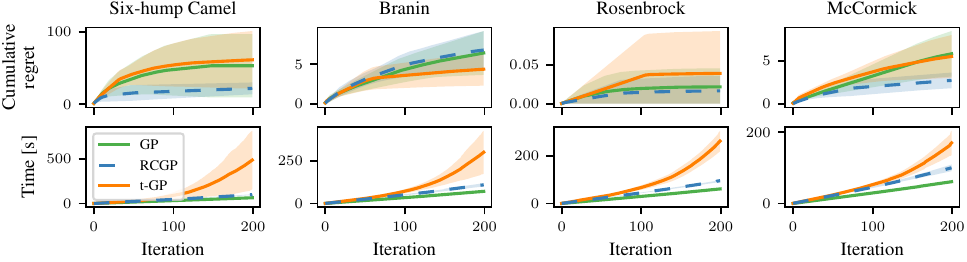}
     \vspace*{-0.3cm}
    \caption{Mean cumulative regret (top) and clock time (bottom) for BO with 
    GP (\textcolor{green_plot}{green}), RCGP (\textcolor{blue_plot}{blue}) and t-GP (\textcolor{orange}{orange}) with 20\% asymmetric outliers and UCB acquisition function over 10 realisations.
    }
    \label{fig:bo}
    \vspace{-5mm}
\end{figure*}
To illustrate the practical utility of RCGPs, we next analyse the Dow Jones Industrial Average (DJIA) index on the 17/04/2013.
On this day, the Associated Press’ Twitter account was hacked and falsely tweeted that explosions at the White House had injured the president. 
This lead to a rapid sell-off in American stock markets within seconds, quickly followed by an equally fast bounce-back.
This resulted in a few moments after 1pm that day where the DJIA did not accurately reflect the USA's economic realities.
In this sense, it is reasonable to regard the resulting group of extreme observations as outliers.
\Cref{fig:flashcrash} depicts the raw data as well as a GP and RCGP fit.
The figure not only illustrates the robustness of RCGPs, but also sheds further light on relevant trade-offs when choosing $w$.
As the bottom panel shows, most points have weights $\approx 1$. 
There are two exceptions to this: one of these is the group of points past 1pm where the crash occurred. Here, the observations are down-weighted almost all the way to zero. 
This is desirable, and exactly what we would expect RCGP to do.
Perhaps more surprisingly, weights are also significantly smaller than $1$ before 10am.
This is not obviously desirable, and due to the choice of prior mean.
In particular, recall that $w(x_i,y_i)$ is small if $|y_i - m(x_i)|$ is large.
For \Cref{fig:flashcrash},  the constant function $m(x) = \frac{1}{n}\sum_{i=1}^n y_i$ is chosen as the prior mean.
Relative to this prior mean, the first observations are outliers.
While this is not ideal, it is worth noting that this behaviour can be easily remedied by choosing $m$ more carefully. For example, one can fit a simple parametric model for $m$, as is extremely common in the literature. 
See \cref{app:prior_mean} for a further discussion.

\paragraph{Bayesian Optimisation} One of the most successful applications of GPs has been Bayesian optimisation (BO) \citep{Garnett2021}. 
In BO, a data-dependent acquisition function $a:\mathcal{X} \rightarrow \mathbb{R}$ based on the posterior GP  characterises desirable regions of $\mathcal{X}$ to choose the next---often noisy---function evaluations of $f$ from.
Here, misspecification of the noise model is a significant concern: it leads to poor acquisition functions; see e.g. \citet{Martinez-Cantin2018,Bogunovic2020,Kirschner2021}. 
Using RCGPs instead of GPs leads to a BO algorithm that is naturally robust to such misspecification.
We illustrate this using the upper-confidence bound (UCB) and probability of improvement (PI) acquisition functions. 
\begin{proposition}
\label{prop:acquisition_functions}
    Let $f \sim \mathcal{GP}(m,k)$, $\vepsilon \sim \mathcal{N}(0,\sigma^2)$. The UCB and PI acquisition functions for RCGPs are
    \begin{IEEEeqnarray*}{rCl}
a_{\operatorname{UCB-RCGP}}(x_{\star}) & = & \mu^{R}_{\star}  + \lambda (\Sigma^{R}_{\star})^{1/2} \\
    a_{\operatorname{PI-RCGP}}(x_{\star}) & = & \Phi\Big(\big(\mu^{R}_{\star}-f(x_{\text{max}})\big)(\Sigma^{R}_{\star})^{-1/2}\Big) 
\end{IEEEeqnarray*}
where $\Phi$ is the cdf of a standard normal, and $x_{\text{max}}$ is the best solution we have so far.
\end{proposition}
\vspace{-1mm}
Here, we have re-used the notations of \cref{prop:RCGP_derivation}, and full derivations are available in \Cref{app:proof_acquisition_functions}. 
The only paper explicitly studying BO in the presence of outliers is \citet{Martinez-Cantin2018}, which uses a t-GP. 
We thus compare RCGPs to GPs and t-GPs on the classical six-hump camel, Branin, McCormick and Rosenbrock functions (see \cref{app:BO}).
\cref{fig:bo} shows the results if each new function evaluation has a $20\%$ chance of being contaminated by an asymmetric outlier generated as in \cref{fig:outliers}, and the BO uses UCB as the acquisition function.
In terms of cumulative regret, RCGPs outperform GPs. 
While t-GPs can match this, they take orders of magnitude longer to run.
The results remain the same if the PI acquisition function is used (See \cref{app:BO}).
It is notable that even without outliers, RCGPs match or outperform GPs (see \cref{app:BO}). 
This is likely because the GP priors on $f$ are misspecified, and robustness to misspecified error models may help in this setting.

\section{Conclusion}
A major issue with existing GP regression techniques is that they are \textit{either} robust \textit{or} conjugate---never both. 
RCGPs solve this issue through generalised Bayesian inference, and demonstrate that robustness does not require prohibitive computational cost. 
Intriguingly, our experiments also indicate that there is no clear disadvantage from using RCGPs in the absence of misspecification---raising the possibility that RCGPs may become a preferred default choice over GPs in the future.
\vspace{-5mm}
\paragraph{Limitations and future work}
One limitation of our work is its dependency on a well-specified prior mean through the weighting function, as illustrated in \cref{fig:flashcrash}. Despite this, the weighting function offers flexibility for practitioners to select one that best suits their specific problem. For example, a practitioner concerned only with outliers in one direction could choose a weighting function that penalises deviations only in that direction. Additionally, we provide guidelines on selecting a weighting function to ensure that RCGP remains provably outlier robust (see \cref{prop:robustness}).

Another limitation of our method is the inability to obtain a posterior over the kernel hyperparameters due to the use of generalised Bayes. However, it is important to emphasise that the primary focus of our method is on speed and robustness rather than implementing a fully Bayesian procedure. Full Bayesian approaches that account for kernel hyperparameter uncertainty with standard GPs are significantly slower, so they are not our primary comparison. Instead, we focus on conjugacy, which is generally not achievable with a prior on hyperparameters.

Lastly, the flexibility and conjugacy of RCGPs means that there is significant scope to extend their use to settings beyond those presented here---including multi-output GPs \citep{bonilla2007multi, altamirano2022nonstationary}, 
linear-time GPs \citep{Hartikainen2010}, GPs with derivative \citep{Morris1993,Wu2017} or integral data \citep{Yousefi2019,Tanaka2019}, 
deep GPs \citep{damianou2013deep}, transformed GPs \citep{maronas2021transforming}, and probabilistic numerics \citep{Hennig2022}. 

\section*{Acknowledgements}
FXB and JK were supported by the EPSRC grant [EP/Y011805/1].

\section*{Impact statement}
This paper presents work whose goal is to advance the field of statistical Machine Learning. There are many potential societal consequences of our work, none of which we feel must be specifically highlighted here.

\bibliography{Bibliography}
\bibliographystyle{icml2024}

\newpage
\appendix
\onecolumn
\vspace*{-18pt}
\section*{\LARGE \centering Robust and Conjugate Gaussian Process Regression: \\ Supplementary Materials
}
\vspace{8pt}
\vspace{24pt}

Our supplementary material is structured as follows. In \Cref{app:proofs}, we provide the proofs of all our theoretical results, as well as additional results which complement those in the main text. In \Cref{app:experiments}, we provide additional details on our numerical experiments.

\section*{Notation}
In this section, we recap the notation used throughout the paper.

\begin{itemize}
    \item If $g:\X\times\mathcal{Y}\subseteq\R^{d}\times\R\to\R$, then $\nabla_y g = \dfrac{\partial g}{\partial y}$ is the partial derivative of $g$ with respect to $y$.
    \item $X\sim\mathcal{N}(\mu,\Sigma)$ denotes that $X$ is has a multivariate Gaussian distribution with mean $\mu$ and covariance $\Sigma$. Furthermore, $p(x)=\mathcal{N}(x;\mu,\Sigma)$ denotes that $p$ is the density of this multivariate Gaussian distribution.
    \item For some $d$-dimensional vector $\mathbf{v} = (v_1,\ldots,v_d)^\top$, $\diag(\mathbf{v})$ is the $d \times d$ diagonal matrix $D$ so that $D_{ij} = 0$ if $i\neq j$ and $D_{ii} = v_i$. In addition, exponents are applied entry-wise; e.g. $\mathbf{v}^2 = (v_1^2,\ldots,v_d^2)^\top$. 
    \item $\Tr$ denotes the trace operator, which for some $d\times d$ matrix $A$ is given by: $\Tr(A)=\sum_{i=1}^{d}A_{ii}$.
    \item A positive-definite kernel $k:\X \times \X \rightarrow \mathbb{R}$ is a symmetric function, i.e. $k(x,x')=k(x',x)$ for any $x,x'$,  that is positive-definite, i.e. $\forall \{x_1,...,x_n\}\subset \mathcal{X}$, $\forall \{c_1,...,c_n\}\subset\R$ and $\forall n\in\N$: $\sum_{i=1}^{n}\sum_{j=1}^{n}c_{i}c_{j}k(x_i,x_j)\geq 0$, which equivalent that for any $n$ points the $n\times n$ matrix given by $k(x,x)$ is positive-semidefinite.
    According to the Moore-Aronszajn theorem, every positive-definite kernel is also a reproducing kernel, and according to Loeve's theorem, any reproducing kernel is the covariance function of a second-order stochastic process and vice-versa. 
    \item Let $X\sim p$ be a random variable distributed according to the probability density $p$. We define the expectation and the variance of the random variable $X$ as 
    \begin{IEEEeqnarray}{rCl}
        \mathbb{E}_{X\sim p}[X]=\int_{\R^{n}}xp(x)dx, \qquad \mathbb{V}_{X\sim p}[X]=\int_{\R^{n}}(x-\mathbb{E}_{X\sim p}[X])(x-\mathbb{E}_{X\sim p}[X])^{\top}p(x)dx\nonumber
    \end{IEEEeqnarray}
\end{itemize}

\section{Proofs of Theoretical Results}\label{app:proofs}

\subsection{Proof of \Cref{prop:RCGP_derivation}}\label{app:proof_RCGP_derivation}

Recall the loss function defined in \cref{sec:methodology}:
\begin{IEEEeqnarray}{rCl}
    L^{w}_n(\f,\y,\x) & = &\dfrac{1}{n}\sum_{i=1}^{n} \big((w s_{\text{model}})^{2}  +2\nabla_{y}(w^{2} s_{\text{model}})\big)(x_i, y_i). \nonumber 
\end{IEEEeqnarray}
Here, $s_{\text{model}}:\mathcal{X} \times \mathcal{Y} \rightarrow \mathbb{R}$ is the score function, and $w:\mathcal{X}\times\mathcal{Y}\to\R$ is the weighting function. \Cref{prop:RCGP_derivation} gives the posterior and posterior predictive distributions for RCGPs, and we now derive these one-by-one.

\paragraph{Posterior Distribution}
 Firstly, we show that the RCGP posterior has density $p^{w}(\f | \y, \x) = \mathcal{N}(\f; \mu^R,\Sigma^R)$, 
\begin{IEEEeqnarray}{rCl}
\mu^R &=& \m+ \km\left(\km+\sigma^{2}J_{\w}\right)^{-1}\left(\y-\m_\w\right)  \nonumber\nonumber, \qquad \qquad
\Sigma^R =  \km\left(\km+\sigma^{2}J_{\w}\right)^{-1}\sigma^{2}J_{\w} \nonumber,
\end{IEEEeqnarray}
where $\w = (w(x_1,y_1),\ldots,w(x_n,y_n))^\top$, $\m_\w = \m +\sigma^{2}\nabla_{y}\log(\w^{2})$,  $J_{\w} =\diag(\frac{\sigma^{2}}{2}\w^{-2})$.  
\begin{proof}
Firstly, it is worth noting that for $s_{\text{model}}(x, y)  = (f(x)-y)\sigma^{-2}$, the loss function $L^{\widetilde{w}}_n$ constructed with the weight function $w$ can be written as 
\begin{IEEEeqnarray}{rCl}
    &L&^{w}_n(\f,\y,\x) =\dfrac{1}{n}\sum_{i=1}^{n} \dfrac{w(x_{i},y_{i})^{2}(f(x_{i})-y_i)^2}{\sigma^{4}} + 2\nabla_{y} \left(\dfrac{w(x_{i},y_{i})^{2}(f(x_{i})-y_i)}{\sigma^{2}}\right) \nonumber\\
     & = &\dfrac{1}{n}\sum_{i=1}^{n} \dfrac{w(x_{i},y_{i})^{2}(f(x_{i})^{2}-2f(x_{i})y_{i}+y_{i}^2)}{\sigma^{4}} + 2\nabla_{y}\left( \dfrac{w(x_{i},y_{i})^{2}(f(x_{i})-y_i)}{\sigma^{2}}\right) \nonumber \\
    &=& \dfrac{1}{n}\sum_{i=1}^{n} \dfrac{w(x_{i},y_{i})^{2}(f(x_{i})^{2}-2f(x_{i})y_{i})}{\sigma^{4}} + 2\nabla_{y} \left(\dfrac{w(x_{i},y_{i})^{2}f(x_{i})}{\sigma^{2}}\right) + \dfrac{w(x_{i},y_{i})^{2}y_{i}^{2}}{\sigma^{4}}- 2\nabla_{y} \left(\dfrac{w(x_{i},y_{i})^{2}y_{i}}{\sigma^{2}}\right).\nonumber
\end{IEEEeqnarray}
Let us consider $\widetilde{w}(x,y) = \sqrt{2\sigma^{-2}}w(x,y)$, then
\begin{IEEEeqnarray}{rCl}
    &L&^{\widetilde{w}}_n(\f,\y,\x) = \dfrac{1}{n}\sum_{i=1}^{n} \dfrac{\widetilde{w}(x_{i},y_{i})^{2}(f(x_{i})^{2}-2f(x_{i})y_{i})}{2\sigma^{2}} + f(x_{i})\nabla_{y} \widetilde{w}(x_{i},y_{i})^{2} + \dfrac{\widetilde{w}(x_{i},y_{i})^{2}y_{i}^{2}}{2\sigma^{2}}- \nabla_{y}( \widetilde{w}(x_{i},y_{i})^{2}y_{i})\nonumber\\
    &=& \dfrac{1}{2n}\left(\f^{\top}\sigma^{-2}\diag(\widetilde{\w}^{2})\f - 2\f^{\top} \left(\sigma^{-2}\diag(\widetilde{\w}^{2})\y-\nabla_{y} \widetilde{\w}^{2}\right) + \y^{\top}\sigma^{-2}\diag\widetilde{\w}^{2})\y - 2\nabla_{y} (\y^{\top}\widetilde{\w}^{2})\right) \nonumber \\
    &=& \dfrac{1}{2n}\left(\f^{\top}\sigma^{-2}\diag(\widetilde{\w}^{2})\f - 2\f^{\top} \sigma^{-2}\diag(\widetilde{\w}^{2})\left(\y-\sigma^{2}\nabla_{y} \log(\widetilde{\w}^{2})\right) + \y^{\top}\sigma^{-2}\diag(\widetilde{\w}^{2})\y - 2\nabla_{y}(\y^{\top}\widetilde{\w}^{2})\right) \nonumber \\
    &=&\dfrac{1}{2n}\left(\f^{\top}\sigma^{-2}J_{\w}^{-1}\f  - 2\f^{\top}\sigma^{-2}J_{\w}^{-1}(\y-\m_{\w}+\m) + C(\x,\y, \sigma^2) \right) \nonumber
\end{IEEEeqnarray}
Where we use the fact that $\nabla_y \widetilde{\w}^2 = \sigma^{-2} \diag(\widetilde{\w}^2) \times \sigma^2 \nabla_y \log (\widetilde{\w}^2)$, and 
\begin{IEEEeqnarray}{rCl}
    \m_\w &=& \m+\sigma^{2}\nabla_{y}\log(\widetilde{\w}^{2}) = \m+\sigma^{2}\nabla_{y}\log(\w^{2})\nonumber\\
    J_{\w} &=&\diag(\widetilde{\w}^{-2}) = \diag(\dfrac{\sigma^{2}}{2}\w^{-2})\nonumber\\
    C(\x,\y, \sigma^2) &=& \y^{\top}\sigma^{-2}\diag(\widetilde{\w}^{2})\y - 2\nabla_{y} \y^{\top}\widetilde{\w}^{2} = \y^{\top}\sigma^{-2}\diag(2\sigma^{-2}\w^{2})\y - 4\sigma^2\nabla_{y} \y^{\top}\w^{2}\nonumber
\end{IEEEeqnarray}
One remark is that $C(\x,\y, \sigma^2)$ does not depend on $\f$. Thus, it will not have an impact on the posterior.

Now, we can calculate the density of the generalised posterior of $\f$ using the loss function defined before as follows
\begin{IEEEeqnarray}{rCl}
    p_{w}(\f | \y) &\propto& p(\f) \exp\{ -n L^{w}_n(\f,\y,\x)\}\nonumber\\
    & \propto & \exp\left(-\dfrac{1}{2}(\f-\m)^{\top}\km^{-1}(\f-\m)\right)\exp\left(-\dfrac{1}{2} \left( \f^{\top}\sigma^{-2}J_{\w}^{-1}\f  - 2\f^{\top}\sigma^{-2}J_{\w}^{-1}(\y-\m_{\w}+\m)\right)\right) \nonumber \\
    &=& \exp\left(-\dfrac{1}{2}\left((\f-\m)^{\top}\km^{-1}(\f-\m) + \f^{\top}\sigma^{-2}J_{\w}^{-1}\f  - 2\f^{\top}\sigma^{-2}J_{\w}^{-1}(\y-\m_{\w}+\m)\right)\right) \nonumber \\
    &\propto & \exp\left(-\dfrac{1}{2}\left(\f^{\top}\km^{-1}\f - 2\f^{\top}\km^{-1}\m + \f^{\top}\sigma^{-2}J_{\w}^{-1}\f  - 2\f^{\top}\sigma^{-2}J_{\w}^{-1}(\y-\m_{\w}+\m)\right)\right) \nonumber \\
    &\propto& \exp\left(-\dfrac{1}{2}\left(\f^{\top}(\km^{-1}+\sigma^{-2}J_{\w}^{-1})\f - 2\f^{\top}\left((\km^{-1}+\sigma^{-2}J_{\w}^{-1})\m + \sigma^{-2}J_{\w}^{-1}(\y-\m_{\w})\right)\right)\right). \nonumber 
\end{IEEEeqnarray}
By completing squares, the posterior has the form
\begin{IEEEeqnarray}{rCl}
p_{w}(\f | \y) &\propto& \exp\left(-\dfrac{1}{2}\left((\f-\mu_R)^{\top}\Sigma_R^{-1}(\f-\mu_R)\right)\right), \nonumber \\
\Sigma_R &=&  (\km^{-1}+\sigma^{-2}J_{\w}^{-1})^{-1} = \km\left(\km+\sigma^{2}J_{\w})\right)^{-1}\sigma^{2}J_{\w}, \nonumber \\
\mu_R &=&  \Sigma_{R}\left((\km^{-1}+\sigma^{-2}J_{\w}^{-1})\m + \sigma^{-2}J_{\w}^{-1}(\y-\m_{\w})\right) = \m +\km\left(\km+\sigma^{2}J_{\w}\right)^{-1}(\y-m_{\w}),\nonumber 
\end{IEEEeqnarray}
where we use the fact that for two invertible matrices $A, B$, we have $(A^{-1}+B^{-1})^{-1} =A(A+B)^{-1}B$. One remark is that $\Sigma_R$ is positive semidefinite, since is the inverse of a sum of positive semidefinite matrices.
\end{proof}
\paragraph{Predictive distribution}
The RCGP posterior predictive distribution over $f_{\star} = f(x_{\star})$ at $x^{\star} \in \mathcal{X}$ is a multivariate Gaussian distribution with density
\begin{IEEEeqnarray}{rCl}
    p^{w}(f_{\star}|x_{\star},\x,\y) & = & \int_{\mathbb{R}^n} p(f_{\star}|x_{\star},\f,\x,\y)p^{w}(\f|\y,\x)d\f  \mathcal{N}(f_{\star};\mu_{\star}^{R},\Sigma_{\star}^{R}),\nonumber
\end{IEEEeqnarray}
\vspace{-8mm}
\begin{IEEEeqnarray}{rCl}
    \mu_{\star}^{R}
    &=&m_\star + \kv_{\star}^{\top}\left(\km+\sigma^{2}J_{\w}\right)^{-1}\left(\y-\m_\w\right),\nonumber\\
    \Sigma_{\star}^{R}
    &=& k_{\star\star}-\kv_{\star}^{\top}(\km+\sigma^{2}J_{\w})^{-1}\kv_{\star}.\nonumber
\end{IEEEeqnarray}
\begin{proof}
    We first derive the predictive for $m(x)=0$ and then extend it to an arbitrary prior mean $m$. In order to compute the predictive, we first need the conditional density $p(f_{\star}|x_{\star},\f,\x,\y)$. Using the fact that $f$ is a mean-zero GP, the joint distribution of $\f$ and $f_{\star}$ is
\begin{IEEEeqnarray}{rCl}
    \begin{pmatrix}
    \f \\
    f_{\star}
    \end{pmatrix}& \sim & \mathcal{N}\left(0,\begin{pmatrix}
    \km & \kv_{\star} \\
    \kv_{\star}^{\top} & k_{\star\star}
    \end{pmatrix}\right)
    \nonumber
\end{IEEEeqnarray}
where $\kv_{\star} = (k(x_{\star},x_1),...,k(x_{\star},x_n))^{\top}$, and $k_{\star,\star} = k(x_\star,x_{\star})$. Therefore, the density of the conditional distribution of a multivariate normal is well known and has the form $p(f_{\star}|x_{\star},\f,\x,\y) = \mathcal{N}(f_{\star};\mu,\Sigma)$ where
\begin{IEEEeqnarray}{rCl}
    \mu &=& \kv_{\star}^{\top}\km^{-1}\f,\nonumber\\
   \Sigma &=& k_{\star\star}-\kv_{\star}^{\top}\km^{-1}\kv_{\star}. \nonumber
\end{IEEEeqnarray}
Let us define $a=\km^{-1}\kv_{\star}$, then $ \mu = a^{\top}\f$, and we can write the density of the predictive distribution as follows
\begin{IEEEeqnarray}{rCl}
    p^{w}(f_{\star}|x_{\star},\x,\y)   & = & \int_{\mathbb{R}^n} p(f_{\star}|x_{\star},\f,\x,\y)p^{w}(\f|\y,\x)d\f. \nonumber\\
    &\propto& \int_{\mathbb{R}^n} \exp\left({-\frac{1}{2}\left((f_{\star}-a^{\top}\f)^{\top}\Sigma^{-1}(f_{\star}-a^{\top}\f)+(\f-\mu^R)^{\top}(\Sigma^{R})^{-1}(\f-\mu^R)\right)}\right)d\f \nonumber \\
    &\propto&\int_{\mathbb{R}^n} \exp\left({-\frac{1}{2}\left(f_{\star}\Sigma^{-1} f_{\star}-2\f^{\top}a\Sigma^{-1}f_{\star}+\f^{\top}a\Sigma^{-1}a^{\top}\f+\f^{\top}(\Sigma^{R})^{-1}\f-2\f^{\top}(\Sigma^{R})^{-1}\mu^{R}\right)}\right)d\f\nonumber\\
    &\propto&\exp\left(-\frac{1}{2}f_{\star}\Sigma^{-1} f_{\star}\right)\int_{\mathbb{R}^n}\exp\left({-\frac{1}{2}\left(\f^{\top}((\Sigma^{R})^{-1}+a\Sigma^{-1}a^{\top})\f-2\f^{\top}\left(a\Sigma^{-1}f_{\star}+(\Sigma^{R})^{-1}\mu^{R}\right)\right)}\right)d\f\nonumber.
\end{IEEEeqnarray}
where the steps follow from basic arithmetic and taking out all terms which do not depend on $\f$, and $\propto$ indicates we do not consider the normalisation constants. Integrating over $\f$, we get
\begin{IEEEeqnarray}{rCl}
     &p^{w}&(f_{\star}|x_{\star},\x,\y) \nonumber\\
    &\propto&\exp\left(-\frac{1}{2}\left(f_{\star}\Sigma^{-1} f_{\star}-(a\Sigma^{-1}f_{\star}+(\Sigma^{R})^{-1}\mu^{R})^{\top}((\Sigma^{R})^{-1}+a\Sigma^{-1}a^{\top})^{-1}(a\Sigma^{-1}f_{\star}+(\Sigma^{R})^{-1}\mu^{R})\right)\right)\nonumber\\
    &\propto&\exp\left(-\frac{1}{2}\left(f^{\star\top}(\Sigma^{-1}-\Sigma^{-1}a^{\top}((\Sigma^{R})^{-1}+a\Sigma^{-1}a^{\top})^{-1}a\Sigma^{-1})f_{\star}-2f^{\star\top}\Sigma^{-1}a^{\top}((\Sigma^{R})^{-1}+a\Sigma^{-1}a^{\top})^{-1}(\Sigma^{R})^{-1}\mu^{R} \right)\right)\nonumber.
\end{IEEEeqnarray}
Therefore, by completing squares, we obtain $p(f_{\star}|x_{\star},\x,\y) = \mathcal{N}(f_{\star};\mu_{\star}^{R},\Sigma_{\star}^{R})$, 
where,
\begin{IEEEeqnarray}{rCl}
    \mu_{\star}^{R}&=&(1-a^{\top}((\Sigma^{R})^{-1}+a\Sigma^{-1}a^{\top})^{-1}a\Sigma^{-1})^{-1}\Sigma \Sigma^{-1}a^{\top}((\Sigma^{R})^{-1}+a\Sigma^{-1}a^{\top})^{-1}(\Sigma^{R})^{-1}\mu^{R},\nonumber \\
    \Sigma_{\star}^{R}&=&(1-a^{\top}((\Sigma^{R})^{-1}+a\Sigma^{-1}a^{\top})^{-1}a\Sigma^{-1})^{-1}\Sigma\nonumber.
\end{IEEEeqnarray}

Now, we expand the terms by arithmetic rules for matrix-vector multiplication, to obtain the final expressions:
\begin{IEEEeqnarray}{rCl}
    \mu_{\star}^{R}&=& (1-a^{\top}((\Sigma^{R})^{-1}+a\Sigma^{-1}a^{\top})^{-1}a\Sigma^{-1})^{-1}\Sigma \Sigma^{-1}a^{\top}((\Sigma^{R})^{-1}+a\Sigma^{-1}a^{\top})^{-1}(\Sigma^{R})^{-1}\mu^{R}\nonumber\\
    &=& (1-a^{\top}((\Sigma^{R})^{-1}+a\Sigma^{-1}a^{\top})^{-1}a\Sigma^{-1})^{-1}(((\Sigma^{R})^{-1}+a\Sigma^{-1}a^{\top})(a^{\top})^{-1})(\Sigma^{R})^{-1}\mu^{R}\nonumber\\
    &=&(((\Sigma^{R})^{-1}+a\Sigma^{-1}a^{\top})(a^{\top})^{-1}-a\Sigma^{-1})^{-1}(\Sigma^{R})^{-1}\mu^{R}\nonumber\\
    &=&((\Sigma^{R})^{-1}(a^{\top})^{-1}+a\Sigma^{-1}-a\Sigma^{-1})^{-1}(\Sigma^{R})^{-1}\mu^{R}\nonumber\\
    &=&((\Sigma^{R})^{-1}(a^{\top})^{-1})^{-1}(\Sigma^{R})^{-1}\mu^{R}\nonumber\\
    &=&a^{\top}\mu^{R}\nonumber\\
    &=&\kv_{\star}^{\top}\left(\km+\sigma^{2}J_{\w}\right)^{-1}(\y-m_{\w})\nonumber.
\end{IEEEeqnarray}

The covariance follows the form:
\begin{IEEEeqnarray}{rCl}
    \Sigma_{\star}^{R}&=& (\Sigma^{-1}-\Sigma^{-1}a^{\top}((\Sigma^{R})^{-1}+a\Sigma^{-1}a^{\top})^{-1}a\Sigma^{-1})^{-1}\nonumber,
\end{IEEEeqnarray}
we use the Woodbury matrix identity on the term $((\Sigma^{R})^{-1}+a\Sigma^{-1}a^{\top})^{-1}$ and by arithmetic rules for matrix-vector multiplication we obtain:
\begin{IEEEeqnarray}{rCl}
   \Sigma_{\star}^{R} &=& (\Sigma^{-1}-\Sigma^{-1}a^{\top}(\Sigma^{R}-\Sigma^{R}a(\Sigma+a^{\top}\Sigma^{R}a)^{-1}a^{\top}\Sigma^{R})a\Sigma^{-1})^{-1}\nonumber\\
    &=& (\Sigma^{-1}-\Sigma^{-1}a^{\top}(\Sigma^{R}a(\Sigma+a^{\top}\Sigma^{R}a)^{-1}((\Sigma+a^{\top}\Sigma^{R}a)a^{-1}-a^{\top}\Sigma^{R}))a\Sigma^{-1})^{-1}\nonumber\\
    &=& (\Sigma^{-1}-\Sigma^{-1}a^{\top}(\Sigma^{R}a(\Sigma+a^{\top}\Sigma^{R}a)^{-1}(\Sigma a^{-1}+a^{\top}\Sigma^{R}-a^{\top}\Sigma^{R})a\Sigma^{-1})^{-1}\nonumber\\
    &=& (\Sigma^{-1}-\Sigma^{-1}a^{\top}(\Sigma^{R}a(\Sigma+a^{\top}\Sigma^{R}a)^{-1}\Sigma a^{-1}a\Sigma^{-1})^{-1}\nonumber\\
    &=& (\Sigma^{-1}-\Sigma^{-1}a^{\top}\Sigma^{R}a(\Sigma+a^{\top}\Sigma^{R}a)^{-1})^{-1}\nonumber\\
    &=& (\Sigma+a^{\top}\Sigma^{R}a)(\Sigma^{-1}(\Sigma+a^{\top}\Sigma^{R}a)-\Sigma^{-1}a^{\top}\Sigma^{R}a)^{-1}\nonumber\\&=&\Sigma+a^{\top}\Sigma^{R}a\nonumber.
\end{IEEEeqnarray}
Finally, we replace $\Sigma^{R}$ and use the Woodbury matrix identity on it, to get:
\begin{IEEEeqnarray}{rCl}
    \Sigma_{\star}^{R} &=& \Sigma+\kv_{\star}^{\top}\km^{-1}(\km^{-1}+\sigma^{-2}J_{\w}^{-1})^{-1}\km^{-1}\kv_{\star}\nonumber\\
    &=& \Sigma+\kv_{\star}^{\top}\km^{-1}(\km-\km(\km+\sigma^{2}J_{\w})^{-1}\km)\km^{-1}\kv_{\star}\nonumber\\
    &=& \Sigma+\kv_{\star}^{\top}\km^{-1}\kv_{\star}-\kv_{\star}^{\top}(\km+\sigma^{2}J_{\w})^{-1}\kv_{\star}\nonumber\\
    &=& k_{\star\star}-\kv_{\star}^{\top}(\km+\sigma^{2}J_{\w})^{-1}\kv_{\star}\nonumber.
\end{IEEEeqnarray}
If the prior mean function $m:\X \rightarrow \mathbb{R}$ is non-zero, we only need a minor modification of the derivation above. This involves recognising that if we have a function $f\sim\GP(m,k)$, then $f'= f-m \sim f\sim\GP(0,k)$. Therefore, when we have observed values of $f$, we can adjust them by subtracting the corresponding prior mean function values, thus obtaining observations of $f'$. We then perform inference on $f'$, and once we have the posterior on this function, we can simply add the prior mean back to the posterior mean to obtain the posterior estimate for $f$. In other words, for a general prior mean $m$ the density of the predictive posterior will be:
\begin{IEEEeqnarray}{rCl}
    p^{w}(f_{\star}|x_{\star},\x,\y) & = & \mathcal{N}(f_{\star};\mu_{\star}^{R},\Sigma_{\star}^{R}),\nonumber
\end{IEEEeqnarray}
\vspace{-8mm}
\begin{IEEEeqnarray}{rCl}
    \mu_{\star}^{R}
    &=&m_\star + \kv_{\star}^{\top}\left(\km+\sigma^{2}J_{\w}\right)^{-1}\left(\y-\m_\w\right),\nonumber\\
    \Sigma_{\star}^{R}
    &=& k_{\star\star}-\kv_{\star}^{\top}(\km+\sigma^{2}J_{\w})^{-1}\kv_{\star}.\nonumber
\end{IEEEeqnarray}
\end{proof}
\subsection{Pseudo Marginal Likelihood for RCGPs}\label{app:proof_RCGP_marginal_likelihood}

While maximising the pseudo marginal likelihood for RCGP $p^{w}(\y| \theta,\sigma^2)$ is ill-posed, it is available in closed form given below.

\begin{proposition}
    The pseudo marginal likelihood for RCGPs takes the form
    \begin{IEEEeqnarray}{rCl}
p^{w}(\y&|& \x,\theta,\sigma^2)\nonumber \\
&=&  \dfrac{1}{\sqrt{ |\km||\km^{-1}+\sigma^{-2}J_{\w}^{-1}|}} \exp\left(\dfrac{1}{2}(\y-m_{\w})^{\top}\sigma^{-2}J_{\w}^{-1}(\km^{-1}+\sigma^{-2}J_{\w}^{-1})^{-1}\sigma^{-2}J_{\w}^{-1}(\y-m_{\w}) - C(\x,\y, \sigma^2) \right) \nonumber.
     \end{IEEEeqnarray}
\end{proposition}
\begin{proof}
    \begin{IEEEeqnarray}{rCl}
    p^{w}(\y| \x,\theta,\sigma^2) &=& \int_{\R^n} p(\f | \x, \theta,\sigma^2)  \exp\{ -n L_n^{w}(\f, \y, \x)\} d\f \nonumber\\
    &=&   \dfrac{1}{\sqrt{(2\pi)^n |\km|}}\int_{\R^n} \exp\left(-\dfrac{1}{2}\f^{\top}\km^{-1}\f -\dfrac{1}{2}\f^{\top}\sigma^{-2}J_{\w}^{-1}\f  + \f^{\top}\sigma^{-2}J_{\w}^{-1}(\y-m_{\w})-C(\x,\y, \sigma^2)\right) d\f\nonumber \\
    &=&   \dfrac{1}{\sqrt{(2\pi)^n |\km|}}\int_{\R^n} \exp\left(-\dfrac{1}{2}\f^{\top}(\km^{-1}+\sigma^{-2}J_{\w}^{-1})\f + \f^{\top}\sigma^{-2}J_{\w}^{-1}(\y-m_{\w}) - C(\x,\y, \sigma^2) \right) d\f. \nonumber 
\end{IEEEeqnarray}
Integrating over $\f$ now yields
\begin{IEEEeqnarray}{rCl}
    &&p^{w}(\y| \x,\theta,\sigma^2) \nonumber\\&=& \dfrac{1}{\sqrt{(2\pi)^n |\km|}} \dfrac{\sqrt{(2\pi)^n}}{ \sqrt{|\km^{-1}+\sigma^{-2}J_{\w}^{-1}|}} \exp\left(\dfrac{1}{2}(\y-m_{\w})^{\top}\sigma^{-2}J_{\w}^{-1}(\km^{-1}+\sigma^{-2}J_{\w}^{-1})^{-1}\sigma^{-2}J_{\w}^{-1}(\y-m_{\w}) - C(\x,\y, \sigma^2) \right) \nonumber\\
     &=&   \dfrac{1}{\sqrt{ |\km||\km^{-1}+\sigma^{-2}J_{\w}^{-1}|}} \exp\left(\dfrac{1}{2}(\y-m_{\w})^{\top}\sigma^{-2}J_{\w}^{-1}(\km^{-1}+\sigma^{-2}J_{\w}^{-1})^{-1}\sigma^{-2}J_{\w}^{-1}(\y-m_{\w}) - C(\x,\y, \sigma^2) \right). \nonumber
\end{IEEEeqnarray}
\end{proof}
\subsection{Leave-one-out Cross-validation Predictive}
\label{app:loo}
The hyperparameters obtained by the leave-one-out cross-validation predictive posteriors are
\begin{IEEEeqnarray}{rCl}
    \hat{\sigma}^2, \hat{\theta} & = &
    \argmax_{\sigma^2, \theta}\Big\{
    \sum_{i=1}^{n}\log p^{w}(y_i|\x,\y_{-i},\theta,\sigma^2)
    \Big\},
    \nonumber
\end{IEEEeqnarray}
where $\y_{-i}=(y_1,\ldots,y_{i-1},y_{i+1},\ldots,y_{n})$. By \cref{prop:RCGP_derivation}, $p^w(y_{i}|\x,\y_{-i},\theta,\sigma^2)=\mathcal{N}(\mu_{i}^{R}, \sigma_{i}^{R} +\sigma^2)$ with
\begin{IEEEeqnarray}{rCl}
    \mu_{i}^{R} &=& z_{i}+\m_{i}-[\left(\km+\sigma^{2}J_{\w}\right)^{-1}\mathbf{z}]_{i}  
    [(K+\sigma^2 J_\w)^{-1}]_{ii}^{-1},
    \nonumber\\
    \sigma_{i}^{R}&=&
    [(K+\sigma^2 J_\w)^{-1}]_{ii}^{-1} -  \frac{\sigma^4}{2} w(x_i, y_i)^{-2}, 
    \nonumber
\end{IEEEeqnarray}
for $\mathbf{z} = \y-\m_\w$ and $\mathbf{z} = (z_1, \dots z_n)$. 
\begin{proof}
    Without loss of generality, we will derive the predictive for $i=n$, which can be extended for an arbitrary $i\in\{1,...,n\}$ using a permutation matrix. Let $p^w(y_{n}|\x,\y_{-n},\theta,\sigma^2)=\mathcal{N}(\mu_{n}^{R}, \sigma_{n}^{R} +\sigma^2)$ with
\begin{IEEEeqnarray}{rCl}
    \mu_{n}^{R}
    &=&\m_n + \km_{1:n-1;n}^{\top}[\km+\sigma^{2}J_{\w^{c}}]_{1:n-1;1:n-1}^{-1}\mathbf{z}_{1:n-1},\nonumber\\
    \Sigma_{\star}^{R}
    &=& k_{nn}-\km_{1:n-1;n}^{\top}[\km+\sigma^{2}J_{\w^{c}}]_{1:n-1;1:n-1}^{-1}\km_{1:n-1;n}.\nonumber
\end{IEEEeqnarray}
where $[\km+\sigma^{2}J_{\w}]_{1:n-1;1:n-1}$ denotes the submatrix formed from rows $\{1,...,n-1\}$ and columns $\{1,...,n-1\}$,  $\km_{1:n-1;n}$ denotes the submatrix formed from $n^{\text{th}}$ row and columns $\{1,...,n-1\}$, and $k_{nn}$ denotes the element from $n^{\text{th}}$ row and $n^{\text{th}}$ columns. Now, we observe that we can write the matrix $\km+\sigma^{2}J_{\w^{c}}$ as 
\begin{IEEEeqnarray}{rCl}
    \km+\sigma^{2}J_{\w}
    = \begin{pmatrix}
        [\km+\sigma^{2}J_{\w}]_{1:n-1;1:n-1} & \km_{n;1:n-1} \\
        \km_{1:n-1;n} & \km_{nn}+\frac{\sigma^4}{2} w(x_n, y_n)^{-2}
    \end{pmatrix} = \begin{pmatrix}
        A & B \\
        C & D
    \end{pmatrix}  \nonumber
\end{IEEEeqnarray}
where we use $A$, $B$, $C$ and $D$ for clarity in the derivation.
Applying block matrix inversion it is now notationally cumbersome, but it is easy to show that.
\begin{IEEEeqnarray}{rCl}
    (\km+\sigma^{2}J_{\w})^{-1}
    = \begin{pmatrix}
       A^{-1} + A^{-1}B\left(D - CA^{-1}B\right)^{-1}CA^{-1} &
      -A^{-1}B\left(D - CA^{-1}B\right)^{-1} \\
    -\left(D-CA^{-1}B\right)^{-1}CA^{-1} &
       \left(D - CA^{-1}B\right)^{-1}
    \end{pmatrix}.\nonumber
\end{IEEEeqnarray}
Therefore, we have that
\begin{IEEEeqnarray}{rCl}
    [(\km+\sigma^{2}J_{\w})^{-1}]_{nn}&=&
       \left(D - CA^{-1}B\right)^{-1} \nonumber \\
       &=&
       \left(\km_{nn}+\frac{\sigma^4}{2} w(x_n, y_n)^{-2} - \km_{1:n-1;n}^{\top}[\km+\sigma^{2}J_{\w}]_{1:n-1;1:n-1}^{-1}\km_{1:n-1;n}\right)^{-1} \nonumber.
\end{IEEEeqnarray}
This leads to
\begin{IEEEeqnarray}{rCl}
    \km_{nn} - \km_{1:n-1;n}^{\top}[\km+\sigma^{2}J_{\w}]_{1:n-1;1:n-1}^{-1}\km_{1:n-1;n} &=& \dfrac{1}{[(\km+\sigma^{2}J_{\w})^{-1}]_{nn}} - \frac{\sigma^4}{2} w(x_n, y_n)^{-2}\nonumber,
\end{IEEEeqnarray}
which gives us the desired variance:
\begin{IEEEeqnarray}{rCl}
    \sigma_{n}^{R}&=&
    [(K+\sigma^2 J_{\w})^{-1}]_{nn}^{-1} - \frac{\sigma^4}{2} w(x_n, y_n)^{-2}, 
    \nonumber
\end{IEEEeqnarray}
Now, for the mean, we use again the block matrix inversion to get
\begin{IEEEeqnarray}{rCl}
    (\km+\sigma^{2}J_{\w})^{-1}\mathbf{z}
    = \begin{pmatrix}
       A^{-1} + A^{-1}B\left(D - CA^{-1}B\right)^{-1}CA^{-1} &
      -A^{-1}B\left(D - CA^{-1}B\right)^{-1} \\
    -\left(D-CA^{-1}B\right)^{-1}CA^{-1} &
       \left(D - CA^{-1}B\right)^{-1}
    \end{pmatrix}\begin{pmatrix}
     \mathbf{z}_{1:n-1} & \\
    z_{n}
    \end{pmatrix}.\nonumber
\end{IEEEeqnarray}
Therefore, we have that
\begin{IEEEeqnarray}{rCl}
    [(\km+\sigma^{2}J_{\w})^{-1}\mathbf{z}]_{n}&=&
       -\left(D-CA^{-1}B\right)^{-1}CA^{-1}\mathbf{z}_{1:n-1} + \left(D-CA^{-1}B\right)^{-1}z_{n}\nonumber\\
       &=&\left(D-CA^{-1}B\right)^{-1}(z_{n}-CA^{-1}\mathbf{z}_{1:n-1})\nonumber.
\end{IEEEeqnarray}
Noting that $\left(D-CA^{-1}B\right)^{-1} = [(\km+\sigma^{2}J_{\w^{c}})^{-1}]_{nn}$, and replacing $A$ and $C$ by their values we obtain
\begin{IEEEeqnarray}{rCl} 
    [(\km+\sigma^{2}J_{\w})^{-1}\mathbf{z}]_{n}
       &=&[(\km+\sigma^{2}J_{\w})^{-1}]_{nn}(z_{n}-\km_{1:n-1;n}^{\top}[\km+\sigma^{2}J_{\w}]_{1:n-1;1:n-1}^{-1}\mathbf{z}_{1:n-1})\nonumber.
\end{IEEEeqnarray}
Finally, rearranging terms, we note that
\begin{IEEEeqnarray}{rCl} 
    \km_{1:n-1;n}^{\top}[\km+\sigma^{2}J_{\w}]_{1:n-1;1:n-1}^{-1}\mathbf{z}_{1:n-1} = z_{n}-[\left(\km+\sigma^{2}J_{\w}\right)^{-1}\mathbf{z}]_{n}  
    [(K+\sigma^2 J_\w)^{-1}]_{nn}^{-1}\nonumber
\end{IEEEeqnarray}
Leading to the desired mean function:
\begin{IEEEeqnarray}{rCl}
    \mu_{n}^{R} &=& z_{n}+\m_{n}-[\left(\km+\sigma^{2}J_{\w}\right)^{-1}\mathbf{z}]_{n}  
    [(K+\sigma^2 J_\w)^{-1}]_{nn}^{-1},
    \nonumber
\end{IEEEeqnarray}
\end{proof}
\subsection{Proof of \Cref{prop:robustness}}\label{app:proof_robustness}
First, define the contamination of the dataset $D=\{(x_i, y_i)\}_{i=1}^{n}$ by the datum $y_m^c$ as $D^{c}_{m}=(D \setminus \{(x_m, y_m)\}) \cup \{(x_m,y^{c}_{m})\}$ for some $m\in\{1,\ldots,n\}$. Let $\y = (y_1,...,y_n)^{\top}$ and $\y^{c} = (y_1,...,y_{m-1}, y_{m}^{c}, y_{m+1},...,y_{n})^{\top}$. 

\paragraph{PIF for the standard GP}
GP regression has the PIF for some constant $C_1\in \R$.
\begin{IEEEeqnarray}{rCl}
   && \operatorname{PIF}_{\operatorname{GP}}(y^{c}_{m}, D) = C_1 (y_m-y^{c}_{m})^{2}, \nonumber
\end{IEEEeqnarray}
and is not robust: $\operatorname{PIF}_{\text{GP}}(y^{c}_{m}, f, D) \rightarrow   \infty$ as $|y^{c}_{m}| \to  \infty$.
\begin{proof}
    Let $p(\f | D) = \mathcal{N}(\f;\mu, \Sigma)$  and $p(\f | D^{c}_{m})= \mathcal{N}(\f;\mu^{c}, \Sigma_{c})$ the uncontaminated and contaminated standard GP posterior respectively. Here,
\begin{IEEEeqnarray}{lCl}
\mu= \m +\km (\km+\sigma^{2} I_{n})^{-1}(\y-\m) &  \qquad  &\mu_{c} = \m +\km (\km+\sigma^{2} I_{n})^{-1}(\y^{c}-\m)\nonumber\\
\Sigma =  \km (\km+ \sigma^{2} I_{n})^{-1} \sigma^{2} I_{n}& \qquad  &\Sigma_{c}= \km (\km+ \sigma^{2} I_{n})^{-1} \sigma^{2} I_{n} \nonumber.
\end{IEEEeqnarray}
Therefore, the PIF has the form
\begin{IEEEeqnarray}{rCl}
    \operatorname{PIF}_{\operatorname{GP}}(y^{c}_{m}, D) =
  \frac{1}{2}\left(
    \Tr\left(\Sigma_{c}^{-1}\Sigma\right)  - n +
    \left(\mu_{c} - \mu\right)^\mathsf{T} (\Sigma_{c})^{-1}\left(\mu_{c} - \mu\right) +
    \ln\left(\frac{\det\Sigma_{c}}{\det\Sigma}\right)
  \right). \nonumber
\end{IEEEeqnarray}
We observe that $\Sigma_{c} = \Sigma$ since they do not depend on $y$, so that
\begin{IEEEeqnarray}{rCl}
    \Tr\left(\Sigma_{c}^{-1}\Sigma\right)  - n &=& \Tr\left(I_{n}\right)  - n = n - n = 0,\nonumber\\
    \ln\left(\frac{\det(\Sigma_{c})}{\det(\Sigma)}\right) &=& \ln\left(\det(\Sigma_{c})\det(\Sigma^{-1})\right) = \ln\left(\det(\Sigma_{c}\Sigma^{-1})\right) = \ln\left(\det(I_n)\right) = 0. \nonumber
\end{IEEEeqnarray}
This finally leads to the PIF
\begin{IEEEeqnarray}{rCl}
    \operatorname{PIF}_{\operatorname{GP}}(y^{c}_{m}, D) =
  \frac{1}{2}\left(
    \left(\mu_{c} - \mu\right)^\mathsf{T} (\Sigma_{c})^{-1}\left(\mu_{c} - \mu\right)
  \right). \nonumber
\end{IEEEeqnarray}
We now notice that the term $\mu_{c} - \mu $ can be written as
\begin{IEEEeqnarray}{rCl}
    \mu_{c} - \mu &=& (\m +\km (\km+\sigma^{2} I_{n})^{-1}(\y-\m))   - (\m +\km (\km+\sigma^{2} I_{n})^{-1}(\y_{c}-\m))  \nonumber\\
    &=&\km (\km+\sigma^{2} I_{n})^{-1}(\y-\y^{c})\nonumber.
\end{IEEEeqnarray}
Substituting the relevant expressions for $\mu_{c} - \mu$ and $\Sigma_{c}$ above, we find 
\begin{IEEEeqnarray}{rCl}
    \operatorname{PIF}_{\operatorname{GP}}(y^{c}_{m}, D) &=&
  \frac{1}{2}\left(
    (\y-\y^{c})^\mathsf{T}(\km+\sigma^{2} I_{n})^{-1}\km (\km (\km+ \sigma^{2} I_{n})^{-1} \sigma^{2} I_{n})^{-1}\km (\km+\sigma^{2} I_{n})^{-1}(\y-\y^{c})
  \right). \nonumber \\
  &=&
  \frac{1}{2}\left(
    (\y-\y^{c})^\mathsf{T}(\km+\sigma^{2} I_{n})^{-1} \km \sigma^{-2} I_{n}(\y-\y^{c})
  \right). \nonumber 
\end{IEEEeqnarray}
Finally, since $\y$ and $\y^{c}$ are the same except for the $m^{\text{th}}$ element, the above expression is equal to
\begin{IEEEeqnarray}{rCl}
    \operatorname{PIF}_{\operatorname{GP}}(y^{c}_{m}, D) =
  \frac{1}{2}\left(
    \left[(\km+\sigma^{2} I_{n})^{-1} \km \sigma^{-2} I_{n}\right]_{mm}(y_m-y^{c}_m)^2
  \right). \nonumber
\end{IEEEeqnarray}
\end{proof}

\paragraph{PIF for the RCGP} 
For RCGPs with  $\sup_{x,y} w(x,y) < \infty$, it holds that
\begin{IEEEeqnarray}{rCl}
        \operatorname{PIF}_{\operatorname{RCGP}}(y^{c}_{m}, D) & \leq &  C_1 (w(x_{m},y^{c}_{m})^2 y^{c}_{m})^{2} + C_2,
   \nonumber
\end{IEEEeqnarray}
for some constants $C_1, C_2\in \R$. Thus, if $\sup_{x,y}\left\{y \cdot w(x,y)^2\right\} < \infty$,  RCGP is robust since $\sup_{y^{c}_{m}} |\operatorname{PIF}_{\operatorname{RCGP}}(y^{c}_{m},  D)| < \infty$.

\begin{proof}
Without loss of generality, we will prove the bound for $m=n$, which can be extended for an arbitrary $m\in\{1,\ldots,n\}$ using a permutation matrix.
Let $p^{w}(\f | D) = \mathcal{N}(\f; \mu^{R}, \Sigma^{R})$  and $p^{w}(\f | D^{c}_{m})= \mathcal{N}(\f;\mu^{c}_{R}, \Sigma^{c}_{R})$ denote the uncontaminated and contaminated standard posterior respectively. %
Here,
\begin{IEEEeqnarray}{lCl}
\mu^{R}= \m +\km\left(\km+\sigma^{2}J_{\w}\right)^{-1}(\y-m_{\w}) &  \qquad  &\mu^{R}_{c} = \m +\km\left(\km+\sigma^{2}J_{\w^{c}}\right)^{-1}(\y-m_{\w^{c}})\nonumber\\
\Sigma^{R} =  \km\left(\km+\sigma^{2}J_{\w}\right)^{-1}\sigma^{2}J_{\w}& \qquad  &\Sigma^{R}_{c}= \km\left(\km+\sigma^{2}J_{\w^{c}}\right)^{-1}\sigma^{2}J_{\w^{c}} \nonumber.
\end{IEEEeqnarray}
Here $\w^{c} = (w(x_1,y_1),...,w(x_n,y_n^c))^{\top}$. Therefore, the PIF has the form
\begin{IEEEeqnarray}{rCl}
    \operatorname{PIF}_{\operatorname{RCGP}}(y^{c}_{m}, D) =
  \frac{1}{2}\left(
    \underbrace{\Tr\left((\Sigma^{R}_{c})^{-1}\Sigma^{R}\right) - n}_{(1)}  +
    \underbrace{\left(\mu^{R}_{c} - \mu^{R}\right)^\mathsf{T} (\Sigma^{R}_{c})^{-1}\left(\mu^{R}_{c} - \mu^{R}\right)}_{(2)} +
    \underbrace{\ln\left(\frac{\det\Sigma^{R}_{c}}{\det\Sigma^{R}}\right)}_{(3)}
  \right). \nonumber
\end{IEEEeqnarray}
Now, we will get a bound for each term in the PIF. The first term can be bound as
\begin{IEEEeqnarray}{rCl}
    (1) = \Tr\left((\Sigma^{R}_{c})^{-1}\Sigma^{R}\right) - n &=&\Tr\left(\sigma^{-2}J_{\w^{c}}^{-1} \left(\km+\sigma^{2}J_{\w^{c}}\right) \left(\km+\sigma^{2}J_{\w}\right)^{-1}\sigma^{2}J_{\w}\right)- n  \nonumber\\
    &\leq&\Tr\left(\sigma^{-2}J_{\w^{c}}^{-1} \left(\km+\sigma^{2}J_{\w^{c}}\right)\right) \Tr\left(\left(\km+\sigma^{2}J_{\w}\right)^{-1}\sigma^{2}J_{\w}\right)- n  \nonumber,
    \end{IEEEeqnarray}
where we use the fact that for two positive semidefinite matrices $A$, $B$, it holds that $\Tr(AB)\leq\Tr(A)\Tr(B)$. Observing that $\left(\km+\sigma^{2}J_{\w}\right)^{-1}\sigma^{2}J_{\w}$ does not depend on the contamination, we can now write $\widetilde{C}_1 = \Tr(\left(\km+\sigma^{2}J_{\w}\right)^{-1}\sigma^{2}J_{\w})$, so that by using the arithmetic rules of traces, we obtain
\begin{IEEEeqnarray}{rCl}
   (1) &\leq&\Tr\left(\sigma^{-2}J_{\w^{c}}^{-1} \left(\km+\sigma^{2}J_{\w^{c}}\right)\right) \widetilde{C}_1 - n  \nonumber\\
    &=&\Tr\left(\sigma^{-2}J_{\w^{c}}^{-1}\km+ I_n\right) \widetilde{C}_1 - n  \nonumber\\
    &=&(\Tr\left(\sigma^{-2}J_{\w^{c}}^{-1}\km\right) + n) \widetilde{C}_1 - n  \nonumber\\
    &=&\left(\sum_{i=1}^{n}\sigma^{-2}w^{2}(x_i,y_y)\km_{ii} + n\right) \widetilde{C}_1 - n . \nonumber
\end{IEEEeqnarray}
Finally, since  $\sup_{x,y} w(x,y) < \infty$, the entire expression can be bounded by a constant $\widetilde{C}_2$ that does not depend on the contamination $y^c_n$ as
\begin{IEEEeqnarray}{rCl}
   (1) &\leq& \left(\sigma^{-2}\sup_{x,y} w(x,y)\sum_{i=1}^{n}\km_{ii} + n\right) \widetilde{C}_1 - n  = \widetilde{C}_2.\nonumber
\end{IEEEeqnarray}

Next, we tackle the  second term by noting that
\begin{IEEEeqnarray}{rCl}
(2) = \left(\mu^{R}_{c} - \mu^{R}\right)^\mathsf{T} (\Sigma^{R}_{c})^{-1}\left(\mu^{R}_{c} - \mu^{R}\right) \leq \lambda_{\max}((\Sigma^{R}_{c})^{-1}) \|\mu^{R}_{c} - \mu^{R}\|_{2}^{2}\leq \lambda_{\max}((\Sigma^{R}_{c})^{-1}) \|\mu^{R}_{c} - \mu^{R}\|_{1}^{2}, \nonumber
\end{IEEEeqnarray}
where $\lambda_{\max}((\Sigma^{R}_{c})^{-1})$ is the maximum eigenvalue of $(\Sigma^{R}_{c})^{-1}$. Then, expanding $\lambda_{\max}((\Sigma^{R}_{c})^{-1})$ and using Weyl's inequality, we get:
\begin{IEEEeqnarray}{rCl}
\lambda_{\max}((\Sigma^{R}_{c})^{-1}) &=& \lambda_{\max}(\km^{-1}+\sigma^{-2}J_{\w^{c}}^{-1}) \leq \lambda_{\max}(\km^{-1})+\lambda_{\max}(\sigma^{-2}J_{\w^{c}}^{-1})) \nonumber.
\end{IEEEeqnarray}
Since $J_{\w^{c}}^{-1} =\diag((\w^{c})^{2})$ 
, and $\sup_{x,y} w(x,y) < \infty$, it holds that $\lambda_{\max}(\sigma^{-2}J_{\w^{c}}^{-1})) = \widetilde{C}_3 < +\infty$, so that we have
\begin{IEEEeqnarray}{rCl}
\lambda_{\max}((\Sigma^{R}_{c})^{-1}) &=&  \lambda_{\max}(\km^{-1})+\widetilde{C}_3 = \widetilde{C}_4\nonumber.
\end{IEEEeqnarray}
We replace this in the expression for $(2)$ to obtain 
\begin{IEEEeqnarray}{rCl}
(2) &\leq& \widetilde{C}_4\|\mu^{R}_{c} - \mu^{R}\|_{1}^{2} \nonumber\\
&=&  \widetilde{C}_4\|\m +\km\left(\km+\sigma^{2}J_{\w^{c}}\right)^{-1}(\y-m_{\w^{c}}) - \m -\km\left(\km+\sigma^{2}J_{\w}\right)^{-1}(\y-m_{\w})\|_{1}^{2} \nonumber\\
&=&  \widetilde{C}_4\|\km(\left(\km+\sigma^{2}J_{\w^{c}}\right)^{-1}(\y-m_{\w^{c}})-\left(\km+\sigma^{2}J_{\w}\right)^{-1}(\y-m_{\w}))\|_{1}^{2} \nonumber.
\end{IEEEeqnarray}
Applying Cauchy-Schwartz we obtain:
\begin{IEEEeqnarray}{rCl}
(2) &\leq&  \widetilde{C}_4\|\km\|_{F}\|\left(\km+\sigma^{2}J_{\w^{c}}\right)^{-1}(\y-m_{\w^{c}})-\left(\km+\sigma^{2}J_{\w}\right)^{-1}(\y-m_{\w})\|_{1}^{2} \nonumber,
\end{IEEEeqnarray}
where $\|.\|_{F}$ denotes the Frobenius norm.  Now, let's write $\km+\sigma^{2}J_{\w^{c}}$ as the block matrix
\begin{IEEEeqnarray}{rCl}
    \km+\sigma^{2}J_{\w^{c}} &=& \begin{pmatrix}
        [\km+\sigma^{2}J_{\w^{c}}]_{1:n-1;1:n-1} & [\km+\sigma^{2}J_{\w^{c}}]_{n;1:n-1} \\
        [\km+\sigma^{2}J_{\w^{c}}]_{1:n-1;n} & [\km+\sigma^{2}J_{\w^{c}}]_{nn}
    \end{pmatrix} \nonumber\\
    &=& \begin{pmatrix}
        [\km+\sigma^{2}J_{\w^{c}}]_{1:n-1;1:n-1} & \km_{n;1:n-1} \\
        \km_{1:n-1;n} & \km_{nn}+\sigma^{2}w(x_n,y_n^c)^{-2}
    \end{pmatrix}, \nonumber
\end{IEEEeqnarray}
where $[\km+\sigma^{2}J_{\w^{c}}]_{1:n-1;1:n-1}$ denotes the submatrix formed from rows $\{1,...,n-1\}$ and columns $\{1,...,n-1\}$, and $\km_{1:n-1;n}$ denotes the submatrix formed from $n^{\text{th}}$ row and columns $\{1,...,n-1\}$. The second equality holds since $J_{\w^{c}}$ is diagonal. Because we assumed that the contamination is in the $n^{\text{th}}$ term, $[\km+\sigma^{2}J_{\w^{c}}]_{1:n-1;1:n-1}$ = $[\km+\sigma^{2}J_{\w}]_{1:n-1;1:n-1}$, so that
\begin{IEEEeqnarray}{rCl}
    \km+\sigma^{2}J_{\w^{c}}
    = \begin{pmatrix}
        [\km+\sigma^{2}J_{\w}]_{1:n-1;1:n-1} & \km_{n;1:n-1} \\
        \km_{1:n-1;n} & \km_{nn}+\sigma^{2}w(x_n,y_n^c)^{-2}
    \end{pmatrix} = \begin{pmatrix}
        A & B \\
        C & q(y_n^c)
    \end{pmatrix}  \nonumber
\end{IEEEeqnarray}
where we use $A$, $B$, and $C$ for clarity in the derivation and to emphasise that these submatrices do not depend on the contamination and therefore can be treated as constants. 
Applying block matrix inversion, it is now notationally cumbersome but easy to show that
\begin{IEEEeqnarray}{rCl}
    (\km+\sigma^{2}J_{\w^{c}})^{-1}
    = \begin{pmatrix}
        A^{-1} +A^{-1}B q(y_n^c)^{-1}CA^{-1} & -A^{-1}Bq(y_n^c)^{-1} \\
        -q(y_n^c)^{-1}CA^{-1} & q(y_n^c)^{-1}
    \end{pmatrix}.\nonumber
\end{IEEEeqnarray}
We can now use this to rewrite the matrix-vector product 
\begin{IEEEeqnarray}{rCl}
    &&\left(\km+\sigma^{2}J_{\w^{c}}\right)^{-1}(\y-m_{\w^{c}})
    \nonumber \\
    &=& \begin{pmatrix}
        A^{-1} +A^{-1}B q(y_n^c)^{-1}CA^{-1} & -A^{-1}Bq(y_n^c)^{-1} \\
        -q(y_n^c)^{-1}CA^{-1} & q(y_n^c)^{-1}
    \end{pmatrix}\begin{pmatrix}
        \y_{1:n-1}\\
        y_{n}^{c}
    \end{pmatrix} \nonumber \\
    &=& \begin{pmatrix}
        (A^{-1} +A^{-1}B q(y_n^c)^{-1}CA^{-1})\y_{1:n-1} -A^{-1}Bq(y_n^c)^{-1}y_{n}^{c} \\
        -q(y_n^c)^{-1}CA^{-1}\y_{1:n-1} + q(y_n^c)^{-1}y_{n}^{c}
    \end{pmatrix},\nonumber
\end{IEEEeqnarray}
where $\y_{1:n-1} = (y_1,...,y_{n-1})^{\top}$. Replicating  the same steps for the uncontaminated terms, we similarly obtain:
\begin{IEEEeqnarray}{rCl}
    &&\left(\km+\sigma^{2}J_{\w}\right)^{-1}(\y-m_{\w})
    \nonumber \\
    &=& \begin{pmatrix}
        (A^{-1} +A^{-1}B q(y_n)^{-1}CA^{-1})\y_{1:n-1} -A^{-1}Bq(y_n)^{-1}y_{n} \\
        -q(y_n)^{-1}CA^{-1}\y_{1:n-1} + q(y_n)^{-1}y_{n}
    \end{pmatrix}.\nonumber
\end{IEEEeqnarray}
Now, we can write
\begin{IEEEeqnarray}{rCl}
    &&\sum_{i=1}^{n-1}\Big|\big(\left(\km+\sigma^{2}J_{\w^{c}}\right)^{-1}(\y-m_{\w^{c}})-\left(\km+\sigma^{2}J_{\w}\right)^{-1}(\y-m_{\w})\big)_{i}\Big| \nonumber\\
    &=&\sum_{i=1}^{n-1}\Big|\big(A^{-1}B q(y_n^c)^{-1}CA^{-1}\y_{1:n-1} -A^{-1}Bq(y_n^c)^{-1}y_{n}^{c} - A^{-1}B q(y_n)^{-1}CA^{-1}\y_{1:n-1} + A^{-1}Bq(y_n)^{-1}y_{n}\big)_{i}\Big|.\nonumber
\end{IEEEeqnarray}
Using the triangle inequality, we can bound this as
\begin{IEEEeqnarray}{rCl}
    &\leq&\sum_{i=1}^{n-1}\Big|\big(A^{-1}B q(y_n^c)^{-1}CA^{-1}\y_{1:n-1}\big)_{i}\Big| +\Big|\big(A^{-1}Bq(y_n^c)^{-1}y_{n}^{c}\big)_{i}\Big|\nonumber \\  &&\qquad+ \underbrace{\Big|\big(A^{-1}B q(y_n)^{-1}CA^{-1}\y_{1:n-1}\big)_{i}\Big| + \Big|\big(A^{-1}Bq(y_n)^{-1}y_{n}\Big|}_{\widetilde{C}_5}\nonumber\\
    &=&\widetilde{C}_5 + \sum_{i=1}^{n-1}\Big|q(y_n^c)^{-1}\Big|\Big|\big(A^{-1}BCA^{-1}\y_{1:n-1}\big)_{i}\Big| +\Big|q(y_n^c)^{-1}y_{n}^{c}\Big|\Big|\big(A^{-1}B\big)_{i}\Big| \nonumber\\
    &=&\widetilde{C}_5 + \Big|q(y_n^c)^{-1}\Big|\sum_{i=1}^{n-1}\Big|\big(A^{-1}BCA^{-1}\y_{1:n-1}\big)_{i}\Big| +\Big|q(y_n^c)^{-1}y_{n}^{c}\Big|\sum_{i=1}^{n-1}\Big|\big(A^{-1}B\big)_{i}\Big| \nonumber
\end{IEEEeqnarray}
where $\widetilde{C}_5$ are all terms that do not depend on the contamination. Now, we can observe that:
\begin{IEEEeqnarray}{rCl}
     \Big|q(y_n^c)^{-1}\Big|= \frac{1}{\Big|\km_{nn}+\sigma^{2}w(x_n,y_n^c)^{-2}\Big|}= \frac{\sigma^{-2}w(x_n,y_n^c)^{2}}{\Big|\km_{nn}\sigma^{-2}w(x_n,y_n^c)^{2}+1\Big|}\leq\sigma^{-2}w(x_n,y_n^c)^{2},\nonumber
\end{IEEEeqnarray}   
where the last inequality holds because $\km_{nn}\sigma^{-2}w(x_n,y_n^c)^{2}>0$. Therefore, since $\sup_{x,y} w(x,y) < \infty$, we can write
\begin{IEEEeqnarray}{rCl}
    &&\widetilde{C}_5 + \Big|q(y_n^c)^{-1}\Big|\sum_{i=1}^{n-1}\Big|\big(A^{-1}BCA^{-1}\y_{1:n-1}\big)_{i}\Big| +\Big|q(y_n^c)^{-1}y_{n}^{c}\Big|\sum_{i=1}^{n-1}\Big|\big(A^{-1}B\big)_{i}\Big| \nonumber\\
    &\leq&  \widetilde{C}_5 + \Big|q(y_n^c)^{-1}\Big|\sum_{i=1}^{n-1}\Big|\big(A^{-1}BCA^{-1}\y_{1:n-1}\big)_{i}\Big| +\Big|q(y_n^c)^{-1}y_{n}^{c}\Big|\sum_{i=1}^{n-1}\Big|\big(A^{-1}B\big)_{i}\Big| \nonumber\\
     &\leq&  \widetilde{C}_6 +\Big|w(x_n,y_n^c)^{2}y_{n}^{c}\Big|\widetilde{C}_7 \nonumber
\end{IEEEeqnarray}
For $\widetilde{C}_6 = \widetilde{C}_5 + \sigma^{-2}\sup_{x,y} w(x,y)^{2}\sum_{i=1}^{n-1}\Big|\big(A^{-1}BCA^{-1}\y_{1:n-1}\big)_{i}\Big|$, and $\widetilde{C}_7=\sigma^{-2}\Big|\sum_{i=1}^{n-1}\Big|\big(A^{-1}B\big)_{i}\Big|$. Now, similarly for the $n^{\text{th}}$ term we have
\begin{IEEEeqnarray}{rCl}
    \Big|\left(\km+\sigma^{2}J_{\w^{c}}\right)^{-1}(\y-m_{\w^{c}})-(\km&+&\sigma^{2}J_{\w})^{-1}(\y-m_{\w})\big)_{n}\Big| \nonumber\\
    &=&\Big|-q(y_n^c)^{-1}CA^{-1}\y_{1:n-1} + q(y_n^c)^{-1}y_{n}^{c}+q(y_n)^{-1}CA^{-1}\y_{1:n-1} - q(y_n)^{-1}y_{n}\Big|\nonumber\\
    &\leq&\Big|q(y_n^c)^{-1}CA^{-1}\y_{1:n-1}\Big| + \Big|q(y_n^c)^{-1}y_{n}^{c}\Big|+\Big|q(y_n)^{-1}CA^{-1}\y_{1:n-1}\Big| + \Big|q(y_n)^{-1}y_{n}\Big|\nonumber\\
    &\leq&\widetilde{C}_8 + \widetilde{C}_9\Big|w(x_n,y_{n}^{c})^2y_{n}^{c}\Big|\nonumber,
\end{IEEEeqnarray}
for $\widetilde{C}_8 = \sigma^{-2}\sup_{x,y} w(x,y)^{2}\Big|CA^{-1}\y_{1:n-1}\Big| +\Big|q(y_n)^{-1}CA^{-1}\y_{1:n-1}\Big| + \Big|q(y_n)^{-1}y_{n}\Big|$ and $\widetilde{C}_9 =  \sigma^{-2}$.
Putting both expressions together, we obtain:
\begin{IEEEeqnarray}{rCl}
(2) &\leq&  \widetilde{C}_4\|\km\|_{F}\|\left(\km+\sigma^{2}J_{\w^{c}}\right)^{-1}(\y-m_{\w^{c}})-\left(\km+\sigma^{2}J_{\w}\right)^{-1}(\y-m_{\w})\|_{1}^{2} \nonumber\\
&\leq&\widetilde{C}_4\|\km\|_{F}\|2((\widetilde{C}_6+\widetilde{C}_8)^2 + (\widetilde{C}_7+\widetilde{C}_9)^2(w(x_n,y_{n}^{c})^2y_{n}^{c})^2)\nonumber\\
&\leq&\widetilde{C}_{10} + \widetilde{C}_{11}(w(x_n,y_{n}^{c})^2y_{n}^{c})^2\nonumber,
\end{IEEEeqnarray}
where $\widetilde{C}_{10}=\widetilde{C}_4\|\km\|_{F}\|2((\widetilde{C}_6+\widetilde{C}_8)^2$, $\widetilde{C}_{11}=\widetilde{C}_4\|\km\|_{F}\|2(\widetilde{C}_7+\widetilde{C}_9)^2$. Lastly, the third and final term can be rewritten using  properties of determinants as
\begin{IEEEeqnarray}{rCl}
    (3) = \ln\left(\frac{\det\Sigma^{R}_{c}}{\det\Sigma^{R}}\right) &=& \ln\left(\frac{\det\left(\left(\km+\sigma^{2}J_{\w^{c}}\right)^{-1}\sigma^{2}J_{\w^{c}}\right)}{\det\left(\left(\km+\sigma^{2}J_{\w}\right)^{-1}\sigma^{2}J_{\w}\right)}\right)\nonumber \\
     &=& \ln\left(\det\left(\left(\km+\sigma^{2}J_{\w}\right)\sigma^{-2}J_{\w}^{-1}\right)\det\left(\left(\km+\sigma^{2}J_{\w^{c}}\right)^{-1}\right)\det\left(\sigma^{2}J_{\w^{c}}\right)\right)\nonumber.
\end{IEEEeqnarray}
Here, we defined $\widetilde{C}_{12} = \det\left(\left(\km+\sigma^{2}J_{\w}\right)\sigma^{-2}J_{\w}^{-1}\right)$ since it does not depend on the contamination, and write
\begin{IEEEeqnarray}{rCl}
    (3) 
     &=& \ln\left(\widetilde{C}_{12}\det\left(\left(\km+\sigma^{2}J_{\w^{c}}\right)^{-1}\right)\det\left(\sigma^{2}J_{\w^{c}}\right)\right)\nonumber\\
     &=& \ln\left(\widetilde{C}_{12}\frac{\det\left(\sigma^{2}J_{\w^{c}}\right)}{\det\left(\km+\sigma^{2}J_{\w^{c}}\right)}\right)\nonumber\\
     &\leq& \ln\left(\widetilde{C}_{12}\frac{\det\left(\sigma^{2}J_{\w^{c}}\right)}{\det\left(\km\right)+\det\left(\sigma^{2}J_{\w^{c}}\right)}\right)\nonumber,
\end{IEEEeqnarray}
where in the last inequality, we use the fact that for two positive semidefinite matrices $A$, $B$, it also holds that $\det(A+B)\geq\det(A) + \det(B)$. Finally, $\det\left(\km\right) > 0$ and $\det\left(\sigma^{2}J_{\w^{c}}\right)>0$, since both are positive definite matrices. Therefore, 
\begin{IEEEeqnarray}{rCl}
\frac{\det\left(\sigma^{2}J_{\w^{c}}\right)}{\det\left(\km\right)+\det\left(\sigma^{2}J_{\w^{c}}\right)} \leq 1,\nonumber
\end{IEEEeqnarray}
which leads  to
\begin{IEEEeqnarray}{rCl}
    (3) 
     &\leq& \ln\left(\widetilde{C}_{12}\right) = \widetilde{C}_{13}.\nonumber
\end{IEEEeqnarray}
Finally, putting the three terms together we obtain the desire bound:

\begin{IEEEeqnarray}{rCl}
    \operatorname{PIF}_{\operatorname{RCGP}}(y^{c}_{m}, D) \leq \widetilde{C}_2 + \widetilde{C}_{10} + \widetilde{C}_{11}(w(x_n,y_{n}^{c})^2y_{n}^{c})^2 +\widetilde{C}_{13} =  C_{1}(w(x_n,y_{n}^{c})^2y_{n}^{c})^2+C_{2}\nonumber
\end{IEEEeqnarray}
where $C_2 = \widetilde{C}_2 + \widetilde{C}_{10} + \widetilde{C}_{11}+\widetilde{C}_{13}$, and $C_1 = \widetilde{C}_{11}$.

\end{proof}
\subsection{Proof of \Cref{prop:ELBO}}
\label{app:proof_ELBO}
For $f \sim \mathcal{GP}(m,k)$, $\vepsilon \sim \mathcal{N}(0,\sigma^2)$, the RCSVGP posterior is $f \sim \mathcal{GP}(\widetilde{\mu}, \widetilde{\Sigma})$, where
\begin{IEEEeqnarray}{rCl}
   \widetilde{\mu}(x) 
   & = &
    \phi_{\u}(x)^{\top} 
    \mu_{\u}, \nonumber \\
    \widetilde{\Sigma}(x, x')
    & = &
    k(x, x') - 
    \phi_{\u}(x)^{\top}
    \left( K_{\u\u} - \Sigma_{\u} \right) \phi_u(x'),
    \nonumber \\
    \mu_{\u} & = &
    \m + K_{\u\u} P_{\u}^{-1} K_{\u}\sigma^{-2}J_{\w}^{-1}(\y-m_{\w}),
    \nonumber \\
    \Sigma_{\u} & = &
     K_{\u\u} P_{\u}^{-1} K_{\u\u},
    \nonumber
\end{IEEEeqnarray}
for  $P_{\u} = \left(K_{\u\u}+K_{\u}^{\top}\sigma^{-2}J_{\w}^{-1}K_{\u}\right)$, $[K_{\u\u}]_{ij} = k(u_i, u_j)$, $[K_{\u}]_{ij} = k(u_i, x_j)$, $[k_{\u}(x)]_i = k(u_i, x)$, and $\phi_{\u}(x) = K_{\u\u}^{-1}k_{\u}(x)$. 
As a function of $\u$, $\theta$, and $\sigma^2$, the corresponding variational objective is
\begin{IEEEeqnarray*}{rCl}
 J(\u,\theta, \sigma^2)&=&   \dfrac{1}{2}\nu^{\top}K_{\u}^{\top}Q_{\u}^{-1}K_{\u}\nu + C(\sigma^2)  + \dfrac{1}{2}\log{\left(\frac{
 \det\left(K_{\u\u}\right)^2}{\det\left(Q_{\u}\right)}\right)} \\&& - \Tr{\Big(\sigma^{-2}J_{\w}^{-\frac{1}{2}}(K-K_{\u}^{\top}K_{\u\u}^{-1}K_{\u})J_{\w}^{-\frac{1}{2}}\Big)}, 
\end{IEEEeqnarray*}
where
$Q_{\u} = K_{\u\u} + K_{\u}^{\top}\sigma^{-2}J_{\w}^{-1}K_{\u}$, $\nu = \sigma^{-2} J_{\w}^{-1}(\y-\m_{\w})$, and
$C(\sigma^2) = C(\x,\y,\sigma^2)$ is a function that depends on the data and $\sigma^2$.
\begin{proof}
Recall that for the standard case, we consider $\f_{\u} = (f(u_1),\ldots,f(u_m))^{\top}$ as the  evaluations corresponding to the inducing input locations $\u = (u_1,...,u_m)^{\top}$. 
Given this, we define the variational distribution as
\begin{IEEEeqnarray}{rcL}
    q(\f, \f_{\u}) = p(\f | \f_{\u}) q(\f_{\u}),\nonumber
\end{IEEEeqnarray}
where $q(\mathbf{u})= \mathcal{N}(\mathbf{u} ; \mu, \Sigma)$. We now seek to approximate the exact posterior $p(\f,\f_\u| \y)$
by the variatonal distribution. We do so by minimising the Kullback-Leibler (KL) divergence  between $q(\f,\f_\u)$
and $p(\f,\f_\u| \y)$. This is equivalent to maximising the ELBO, which is defined as:
\begin{IEEEeqnarray}{rCL}
    \operatorname{ELBO}(\u) = \int \log{\frac{\Psi(\y, \f_{\u}) p(\f_{\u})}{q(\f_{\u})}} q(\f_{\u}) d\f_{\u},\nonumber
\end{IEEEeqnarray}
where 
\begin{IEEEeqnarray}{rCl}
    \Psi(\y, \f_{\u})= \exp\left(\int_{\R^n}\log(p(\y|\f))p(\f|\f_{\u})d\f\right).\nonumber
\end{IEEEeqnarray}
It is straightforward to verify that the optimal variational distribution is
\begin{IEEEeqnarray}{rCl}
    q(\f_\u) =\dfrac{\Psi(\y, \f_{\u}) p(\f_{\u})}{\int\Psi(\y, \f_{\u}) p(\f_{\u})d\f_\u}.\nonumber
\end{IEEEeqnarray}
Plugging this back into the ELBO, we then get 
\begin{IEEEeqnarray}{rCL}
    \operatorname{ELBO}(\u,\mu,\Sigma) = \log\left(\int\Psi(\y, \f_{\u}) p(\f_{\u})d\f_\u\right).\nonumber
\end{IEEEeqnarray}
For the RCSVGP, we replace the standard likelihood $p(\y|\f)$, with the pseudo likelihood $\exp\{ -n L_n^{w}(\f, \y, \x)\}$, leading to the RCSVGP's optimal variational distribution, which is given by
\begin{IEEEeqnarray}{rCl}
    q^{w}(\f_\u) =\dfrac{\Psi^{w}(\y, \f_{\u}) p(\f_{\u})}{\int\Psi^{w}(\y, \f_{\u}) p(\f_{\u})d\f_u},\nonumber
\end{IEEEeqnarray}
where
\begin{IEEEeqnarray}{rCl}
    \log\Psi^{w}(\y, \f_{\u})&=& \int_{\R^n}\log(\exp\{ -n L_n^{w}(\f, \y, \x)\})p(\f|\f_{\u})d\f\nonumber\\
    &=&
    \int_{\R^n} -\dfrac{1}{2}\left(\f^{\top}\sigma^{-2}J_{\w}^{-1}\f  - 2\f^{\top}\nu+ C(\x,\y,\sigma^2) \right)p(\f|\f_{\u})d\f,\nonumber
\end{IEEEeqnarray}
where we used the expression of $L_n^{w}(\f, \y, \x)$ from \cref{app:proof_RCGP_derivation} and defined $\nu= \sigma^{-2}J_{\w}^{-1}(\y-\m_{\w})$. Recallling that $C(\x,\y,\sigma^2) = \y^{\top}\sigma^{-2}\diag(\w^{2})\y - 2\nabla_{y} \y^{\top}\w^{2}$. 

Recalling that the density of the conditional distribution of $\f$ given $\f_\u$ is
\begin{IEEEeqnarray}{rCl}
    p(\f|\f_{\u}) &=& \mathcal{N}(\f; \mu_{\f|\f_{\u}},\Sigma_{\f|\f_{\u}}),\nonumber\\
    \mu_{\f|\f_{\u}}&=& K_{\u}^{\top}K_{\u\u}^{-1}\f_{\u},\nonumber\\
    \Sigma_{\f|\f_{\u}}&=&K-K_{\u}^{\top}K_{\u\u}^{-1}K_{\u},\nonumber
\end{IEEEeqnarray}
 $[K_{\u\u}]_{ij} = k(u_i, u_j)$, $[K_{\u}]_{ij} = k(u_i, x_j)$, $[k_{\u}(x)]_i = k(u_i, x)$. Then, we can write the above expression as: 
 
\begin{IEEEeqnarray}{rCl}
    \log\Psi^{w}(\y, \f_{\u})&=&\int_{\R^n} -\dfrac{1}{2}\f^{\top}\sigma^{-2}J_{\w}^{-1}\f  p(\f|\f_{\u})d\f + \int_{\R^n}\f^{\top}\nu p(\f|\f_{\u})d\f-\int_{\R^n} \dfrac{1}{2}C(\x,\y) p(\f|\f_{\u})d\f\nonumber\\
    &=&-\dfrac{1}{2}\int_{\R^n} \f^{\top}\sigma^{-2}J_{\w}^{-1}\f  p(\f|\f_{\u})d\f + \mathbb{E}_{\f\sim p(\f|\f_\u)}[\f^{\top}\nu] - \dfrac{1}{2}C(\x,\y, \sigma^2)\nonumber
\end{IEEEeqnarray}
Recalling that we can write the inner product between two vectors as the trace of their outer product, in this case
\begin{IEEEeqnarray}{rCl}
   \f^{\top}\sigma^{-2}J_{\w}^{-1}\f = (\sigma^{-1}J_{\w}^{-1/2}\f)^{\top}(\sigma^{-1}J_{\w}^{-1/2}\f) = \Tr(\sigma^{-1}J_{\w}^{-1/2}\f\f^{\top}\sigma^{-1}J_{\w}^{-1/2}).\nonumber
\end{IEEEeqnarray}
Therefore,
\begin{IEEEeqnarray}{rCl}
    \log\Psi^{w}(\y, \f_{\u})
    &=&-\dfrac{1}{2}\int_{\R^n} \Tr(\sigma^{-1}J_{\w}^{-1/2}\f\f^{\top}\sigma^{-1}J_{\w}^{-1/2})  p(\f|\f_{\u})d\f + \mathbb{E}_{\f\sim p(\f|\f_\u)}[\f^{\top}\nu] - \dfrac{1}{2}C(\x,\y, \sigma^2) \nonumber\\
    &=&-\dfrac{1}{2}\mathbb{E}_{\f\sim p(\f|\f_\u)}[\Tr(\sigma^{-1}J_{\w}^{-1/2}\f\f^{\top}\sigma^{-1}J_{\w}^{-1/2})|\f_\u] + \mathbb{E}_{\f\sim p(\f|\f_\u)}[\f^{\top}\nu] - \dfrac{1}{2}C(\x,\y, \sigma^2). \nonumber
\end{IEEEeqnarray}
It is well known that the expectation of the trace is equal to the trace of the expectation so that
\begin{IEEEeqnarray}{rCl}
    \log\Psi^{w}(\y, \f_{\u})
    &=&-\dfrac{1}{2}\Tr(\mathbb{E}_{\f\sim p(\f|\f_\u)}[\sigma^{-1}J_{\w}^{-1/2}\f\f^{\top}\sigma^{-1}J_{\w}^{-1/2}]) + \mathbb{E}_{\f\sim p(\f|\f_\u)}[\f^{\top}\nu] - \dfrac{1}{2}C(\x,\y, \sigma^2) \nonumber\\
    &=&-\dfrac{1}{2}\Tr(\sigma^{-1}J_{\w}^{-1/2}\mathbb{E}_{\f\sim p(\f|\f_\u)}[\f\f^{\top}]\sigma^{-1}J_{\w}^{-1/2}) + \mathbb{E}_{\f\sim p(\f|\f_\u)}[\f]^{\top}\nu - \dfrac{1}{2}C(\x,\y, \sigma^2). \nonumber
\end{IEEEeqnarray}
Recalling that the variance of a random variable is defined for a random variable $X$ distributed by $p$ as
\begin{IEEEeqnarray}{rCl}
\mathbb{V}_{X\sim p}[X] = \mathbb{E}_{X\sim p}[XX^{\top}]-\mathbb{E}_{X\sim p}[X]\mathbb{E}_{X}[X]^{\top}\nonumber,
\end{IEEEeqnarray}
which implies that $ \mathbb{E}_{X\sim p}[XX^{\top}] = \mathbb{V}_{X\sim p}[X]+\mathbb{E}_{X\sim p}[X]\mathbb{E}_{X\sim p}[X]^{\top}$. Therefore, we can use this relationship to write

\begin{IEEEeqnarray}{rCl}
    \log&\Psi^{w}&(\y, \f_{\u})
    =-\dfrac{1}{2}\Tr(\sigma^{-1}J_{\w}^{-1/2}\mathbb{E}_{\f\sim p(\f|\f_\u)}[\f\f^{\top}]\sigma^{-1}J_{\w}^{-1/2}) + \mathbb{E}_{\f\sim p(\f|\f_\u)}(\f)^{\top}\nu - \dfrac{1}{2}C(\x,\y, \sigma^2)\nonumber\\
    &=&-\dfrac{1}{2}\Tr(\sigma^{-1}J_{\w}^{-1/2}(\mathbb{V}_{\f\sim p(\f|\f_\u)}[\f]+ \mathbb{E}_{\f\sim p(\f|\f_\u)}[\f]\mathbb{E}_{\f\sim p(\f|\f_\u)}[\f]^{\top})\sigma^{-1}J_{\w}^{-1/2}) + \mathbb{E}_{\f\sim p(\f|\f_\u)}[\f]^{\top}\nu - \dfrac{1}{2}C(\x,\y, \sigma^2).\nonumber
\end{IEEEeqnarray}
Since $p(\f|\f_{\u}) = \mathcal{N}(\f; \mu_{\f|\f_{\u}},\Sigma_{\f|\f_{\u}})$, we know explicitly $\mathbb{V}_{\f\sim p(\f|\f_\u)}[\f]$ and $\mathbb{E}_{\f\sim p(\f|\f_\u)}[\f]$. Plug these expressions in and rearranging terms using traces properties, we obtain the final expression
\begin{IEEEeqnarray}{rCl}
    \log\Psi^{w}(\y, \f_{\u})
    &=&-\dfrac{1}{2}\Tr(\sigma^{-1}J_{\w}^{-1/2}(\mathbb{V}_{\f\sim p(\f|\f_\u)}[\f]+ \mathbb{E}_{\f\sim p(\f|\f_\u)}[\f]\mathbb{E}_{\f\sim p(\f|\f_\u)}[\f]^{\top})\sigma^{-1}J_{\w}^{-1/2}) + \mathbb{E(\f|\f_\u)}^{\top}\nu - \dfrac{1}{2}C(\x,\y, \sigma^2)\nonumber\\
    &=&-\dfrac{1}{2}\Tr(\sigma^{-1}J_{\w}^{-1/2}(\Sigma_{\f|\f_{\u}}+ \mu_{\f|\f_{\u}}\mu_{\f|\f_{\u}}^{\top})\sigma^{-1}J_{\w}^{-1/2}) + \mu_{\f|\f_{\u}}^{\top}\nu - \dfrac{1}{2}C(\x,\y, \sigma^2)\nonumber\\
    &=&-\dfrac{1}{2}\Tr(\sigma^{-2}J_{\w}^{-1/2}\Sigma_{\f|\f_{\u}}J_{\w}^{-1/2}) -\dfrac{1}{2}\Tr(\sigma^{-2}J_{\w}^{-1/2}\mu_{\f|\f_{\u}}\mu_{\f|\f_{\u}}^{\top}J_{\w}^{-1/2})+ \mu_{\f|\f_{\u}}^{\top}\nu - \dfrac{1}{2}C(\x,\y, \sigma^2)\nonumber\\
    &=&-\dfrac{1}{2}\Tr(\sigma^{-2}J_{\w}^{-1/2}\Sigma_{\f|\f_{\u}}J_{\w}^{-1/2}) -\dfrac{1}{2}\mu_{\f|\f_{\u}}^{\top}\sigma^{-2}J_{\w}^{-1}\mu_{\f|\f_{\u}}+ \mu_{\f|\f_{\u}}^{\top}\nu - \dfrac{1}{2}C(\x,\y, \sigma^2).\nonumber
\end{IEEEeqnarray}
We now will plug $\Psi$ in optimal variational distribution to finish the proof. Replacing and rearranging terms, we obtain
\begin{IEEEeqnarray}{rCl}
    q^{w}(\f_\u) &\propto&\Psi^{w}(\y, \f_{\u}) p(\f_{\u})\nonumber\\
    &=&\exp\left( -\dfrac{1}{2}\mu_{\f|\f_{\u}}^{\top}\sigma^{-2}J_{\w}^{-1}\mu_{\f|\f_{\u}}+ \mu_{\f|\f_{\u}}^{\top}\nu -\dfrac{1}{2}\f_{\u}^{\top}K_{\u\u}^{-1}\f_{\u}\right)\nonumber\\
    &=&\exp\left( -\dfrac{1}{2}\f_{\u}^{\top}K_{\u\u}^{-1}K_{\u}\sigma^{-2}J_{\w}^{-1}K_{\u}^{\top}K_{\u\u}^{-1}\f_{\u}+ \f_{\u}^{\top}K_{\u\u}^{-1}K_{\u}\nu -\dfrac{1}{2}\f_{\u}^{\top}K_{\u\u}^{-1}\f_{\u}\right)\nonumber\\
    &=&\exp\left( -\dfrac{1}{2}\f_{\u}^{\top}(K_{\u\u}^{-1}K_{\u}\sigma^{-2}J_{\w}^{-1}K_{\u}^{\top}K_{\u\u}^{-1}+K_{\u\u}^{-1})\f_{\u}+ \f_{\u}^{\top}K_{\u\u}^{-1}K_{\u}\nu \right)\nonumber
\end{IEEEeqnarray}
completing squares, we obtain
 \begin{IEEEeqnarray}{rCl}
    q^{w}(\f_\u) &\propto&\exp\left( -\dfrac{1}{2}(\f_{\u}-\mu_{\u})^{\top}\Sigma_{S}^{-1}(\f_{\u}-\mu_{\u})^{\top}\right)\nonumber\\
    \mu_{\u}&=& (K_{\u\u}^{-1} + K_{\u\u}^{-1}K_{\u}\sigma^{-2}J_{\w}^{-1}K_{\u}^{\top}K_{\u\u}^{-1})^{-1}K_{\u\u}^{-1}K_{\u}\nu\nonumber\\
    &=& K_{\u\u}(K_{\u\u} + K_{\u}\sigma^{-2}J_{\w}^{-1}K_{\u}^{\top})^{-1}K_{\u}\nu\nonumber\\
    &=& K_{\u\u}P_{\u}^{-1}K_{\u}\sigma^{-2}J_{\w}^{-1}(\y-\m_{\w})\nonumber,\\
    \Sigma_{\u}&=&(K_{\u\u}^{-1}+K_{\u\u}^{-1}K_{\u}\sigma^{-2}J_{\w}^{-1}K_{\u}^{\top}K_{\u\u}^{-1})^{-1}\nonumber\\
    &=&K_{\u\u}(K_{\u\u}+K_{\u}\sigma^{-2}J_{\w}^{-1}K_{\u}^{\top})^{-1}K_{\u\u}\nonumber\\
    &=&K_{\u\u}P_{\u}^{-1}K_{\u\u},\nonumber
\end{IEEEeqnarray}
for  $P_{\u} = \left(K_{\u\u}+K_{\u}^{\top}\sigma^{-2}J_{\w}^{-1}K_{\u}\right)$. Now, for the predictive posterior over $f_{\star} = f(x_{\star})$ at new point $x^{\star}\in\X$, we have
\begin{IEEEeqnarray}{rCl}
    p^{w}(f_{\star}|,\y,\x) &=& \int_{\R}p^{w}(f_{\star}|x_{\star}, \f_{\u})q^{w}(\f_{\u})d\f_{\u}\nonumber\\
    &=& \mathcal{N}(f_{\star}; \mu_{f_{\star}|\u},\Sigma_{f_{\star}|\u})\mathcal{N}(\f_{\u}; \mu_{\u},\Sigma_{\u})d\f_{\u}\nonumber.
\end{IEEEeqnarray}
Here, we can use the integral obtain for the predictive in \cref{app:proof_RCGP_derivation} to get
\begin{IEEEeqnarray}{rCl}
    p^{w}(f_{\star}|x,\y,\x) &=& \mathcal{N}(f_{\star};\widetilde{\mu}(x_{\star}),\widetilde{\Sigma}(x_{\star},x_{\star}))\nonumber\\
    \widetilde{\mu}(x_{\star}) 
   & = &
    \phi_{\u}(x_{\star})^{\top} 
    \mu_{\u}, \nonumber \\
    \widetilde{\Sigma}(x_{\star}, x_{\star})
    & = &
    k(x_{\star}, x_{\star}) - 
    \phi_{\u}(x_{\star})^{\top}
    \left( K_{\u\u} - \Sigma_{\u} \right) \phi_u(x_{\star})
    \nonumber,
\end{IEEEeqnarray}
for $[k_{\u}(x)]_i = k(u_i, x)$, and $\phi_{\u}(x) = K_{\u\u}^{-1}k_{\u}(x)$.
Finally, plugging $\Psi^{w}$ in the ELBO and using log and exponential properties, we get
\begin{IEEEeqnarray}{rCL}
    &&\operatorname{ELBO}(\u) = \log\left(\int\Psi^{w}(\y, \f_{\u}) p(\f_{\u})d\y\right)\nonumber\\
    &=&\log\left(\int_{\R^{n}}\exp\left(-\frac{1}{2}\Tr(\sigma^{-2}J_{\w}^{-1/2}\Sigma_{\f|\f_{\u}}J_{\w}^{-1/2}) -\frac{1}{2}\mu_{\f|\f_{\u}}^{\top}\sigma^{-2}J_{\w}^{-1}\mu_{\f|\f_{\u}}+ \mu_{\f|\f_{\u}}^{\top}\nu - \frac{1}{2}C(\x,\y, \sigma^2)-\frac{1}{2}\f_{\u}^{\top}K_{\u\u}^{-1}\f_{\u}\right)d\f_\u\right)\nonumber\\
    &=&-\frac{1}{2}(\Tr(\sigma^{-2}J_{\w}^{-1/2}(K-K_{\u}^{\top}K_{\u\u}^{-1}K_{\u})J_{\w}^{-1/2})-C(\x,\y, \sigma^2)) \nonumber\\
    &\qquad& +\log\left(\int_{\R^{n}}\exp\left( -\dfrac{1}{2}\f_{\u}^{\top}(K_{\u\u}^{-1}K_{\u}\sigma^{-2}J_{\w}^{-1}K_{\u}^{\top}K_{\u\u}^{-1}+K_{\u\u}^{-1})\f_{\u}+ \f_{\u}^{\top}K_{\u\u}^{-1}K_{\u}\nu \right)d\f_\u\right)\nonumber
\end{IEEEeqnarray}
integrating over $\f_\u$ and just arithmetic rules for matrix-vector multiplication, we obtain
\begin{IEEEeqnarray}{rCL}
    \operatorname{ELBO}(\u) &=&
    -\frac{1}{2}(\Tr(\sigma^{-2}J_{\w}^{-1/2}(K-K_{\u}^{\top}K_{\u\u}^{-1}K_{\u})J_{\w}^{-1/2})-C(\x,\y, \sigma^2)) \nonumber\\
    &\qquad& +\log\left(\dfrac{\exp\left( \dfrac{1}{2}\nu^{\top}K_{\u}^{\top}K_{\u\u}^{-1}(K_{\u\u}^{-1}+ K_{\u\u}^{-1}K_{\u}\sigma^{-2}J_{\w}^{-1}K_{\u}^{\top}K_{\u\u}^{-1})^{-1}K_{\u\u}^{-1}K_{\u}\nu \right)}{\sqrt{\det\left(K_{\u\u}^{-1}+ K_{\u\u}^{-1}K_{\u}\sigma^{-2}J_{\w}^{-1}K_{\u}^{\top}K_{\u\u}^{-1}\right)}}\right)\nonumber\\
    &=&
    -\frac{1}{2}(\Tr(\sigma^{-2}J_{\w}^{-1/2}(K-K_{\u}^{\top}K_{\u\u}^{-1}K_{\u})J_{\w}^{-1/2})-C(\x,\y, \sigma^2)) \nonumber\\
    &\qquad& +\log\left(\frac{\exp\left(\dfrac{1}{2}\nu^{\top}K_{\u}^{\top}(K_{\u\u}+ K_{\u}\sigma^{-2}J_{\w}^{-1}K_{\u}^{\top})^{-1}K_{\u}\nu \right)}{\sqrt{\det\left(K_{\u\u}^{-1}+ K_{\u\u}^{-1}K_{\u}\sigma^{-2}J_{\w}^{-1}K_{\u}^{\top}K_{\u\u}^{-1}\right)}}\right)\nonumber\\
    &=&
    -\frac{1}{2}(\Tr(\sigma^{-2}J_{\w}^{-1/2}(K-K_{\u}^{\top}K_{\u\u}^{-1}K_{\u})J_{\w}^{-1/2})-C(\x,\y, \sigma^2)) \nonumber\\
    &\qquad& +\dfrac{1}{2}\nu^{\top}K_{\u}^{\top}(K_{\u\u}+ K_{\u}\sigma^{-2}J_{\w}^{-1}K_{\u}^{\top})^{-1}K_{\u}\nu+\dfrac{1}{2}\log\left(\frac{\det(K_{\u\u})^2}{\det\left(K_{\u\u}+ K_{\u}\sigma^{-2}J_{\w}^{-1}K_{\u}^{\top}\right)}\right)\nonumber
\end{IEEEeqnarray}

\end{proof}

\subsection{Proof of \Cref{prop:acquisition_functions}
}\label{app:proof_acquisition_functions}
Let $f \sim \mathcal{GP}(m,k)$, $\vepsilon \sim \mathcal{N}(0,\sigma^2)$. The UCB and PI acquisition functions for RCGPs are
\begin{IEEEeqnarray*}{rCl}
a_{\operatorname{UCB-RCGP}}(x_{\star}) & = & \mu^{R}_{\star}  + \lambda (\Sigma^{R}_{\star})^{1/2} \\
a_{\operatorname{PI-RCGP}}(x_{\star}) & = & \Phi\Big(\big(\mu^{R}_{\star}-f(x_{\text{max}})\big)(\Sigma^{R}_{\star})^{-1/2}\Big) 
\end{IEEEeqnarray*}
where $\Phi$ is the cdf of a standard normal, and $x_{\text{max}}$ is the best solution we have so far.
\begin{proof}
Since the RCGP posterior is Gaussian, we can straightforwardly define the upper confidence bound acquisition function as 
\begin{IEEEeqnarray}{rCl}
a_{\operatorname{UCB-RCGP}}(x_{\star}) & = & \mu^{R}_{\star}  + \lambda (\Sigma^{R}_{\star})^{1/2} \nonumber
\end{IEEEeqnarray}   

For probability of improvement, suppose that the best solution we have so far is $x_{\text{max}}$. We define the improvement function as:
\begin{IEEEeqnarray}{rCl}
    \text{I}(x_{\star}) = 
\max(f(x_{\star}) - f(x_{\text{max}}), 0)\nonumber
\end{IEEEeqnarray}
we know that $f(x_{\star})\sim\mathcal{N}(\mu^{R}_{\star}, \Sigma^{R}_{\star})$, therefore, we can rewrite it as $f(x_{\star})= \mu^{R}_{\star} + \Sigma^{R}_{\star} z$, with $z\sim\mathcal{N}(0, 1)$, leading to: 
\begin{IEEEeqnarray}{rCl}
    \text{I}(x_{\star}) = 
\max(\mu^{R}_{\star} + \Sigma^{R}_{\star}z - f(x_{\text{max}}), 0)\quad z\sim\mathcal{N}(0, 1)\nonumber
\end{IEEEeqnarray}
Finally, we define the probabilty of improvement acquisition function as:
\begin{IEEEeqnarray}{rCl}
    a_{\operatorname{PI-RCGP}}(x_{\star}) & = & \mathbb{P}(\text{I}(x_{\star})>0)\nonumber \\
    &=& \mathbb{P}(f(x_{\star})> f(x_{\text{max}}))\nonumber\\
    &=& \mathbb{P}(z> \big(f(x_{\text{max}})-\mu^{R}_{\star}\big)(\Sigma^{R}_{\star})^{-1/2})\nonumber\\
    &=& \Phi\Big(\big(\mu^{R}_{\star}-f(x_{\text{max}})\big)(\Sigma^{R}_{\star})^{-1/2}\Big) \nonumber
\end{IEEEeqnarray}

\end{proof}

\section{Additional Experimental Results}\label{app:experiments}
All the experiments were running on an Apple M2 Pro CPU with 16 GB of memory.

\subsection{Prior mean}\label{app:prior_mean}
This section illustrates how the choice of the prior mean affects our method. In this particular example, we generated data from a Gaussian Process (GP) with a zero mean and a periodic squared exponential kernel, where the length scale and variance are set to 1, and the period is set to 3. We then added 1\% of focused contamination near zero. \cref{fig:better_m} demonstrates the performance of two RCGPs on this dataset: the blue one with a zero prior mean and the brown one with a prior mean selected through fitting a polynomial regression.  Since the outliers in this dataset are close to zero, the method with a zero prior mean does not identify them as outliers. In fact, as the outliers align exactly with the prior mean, the corresponding weight assigned to them is the largest. This leads to a non-optimal performance. In contrast, for the brown RCGP, since the prior mean is carefully chosen, the outliers are far from it and, therefore, are down-weighted, resulting in a posterior that matches the true generative process.
\begin{figure}[h]
    \centering
    \includegraphics{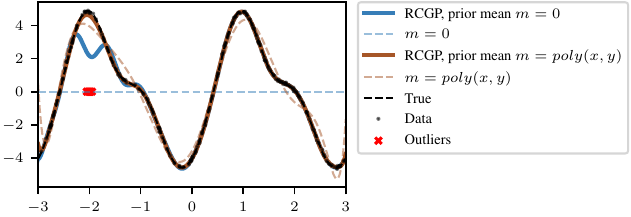}
    \caption{The posterior predictive mean of the RCGP with zero prior mean (\textcolor{blue_plot}{\textbf{blue}}) and the RCGP with a prior mean selected through fitting a polynomial regression (\textcolor{brown}{\textbf{brown}}) on a synthetic dataset where 1\% of the data are focused generated outliers.}
    \label{fig:better_m}
\end{figure}

\subsection{Implementation of Competing Methods}

\paragraph{GP} For the standard GP, we choose hyperparameters via maximum likelihood and utilise the implementation in \texttt{GPflow} \citep{GPflow2017}

\paragraph{t-GP} This model replaces the usual choice of Gaussian observation noise with a Student-t distributed observation noise; i.e.
\begin{IEEEeqnarray}{rCl}
   p(y| f(x), \nu,\sigma^2)&=&\frac{\Gamma(\frac{\nu+1}{2})} {\sqrt{\nu\pi\sigma^2}\,\Gamma(\frac{\nu}{2})} \left(1+\frac{(y-f(x))^2}{\nu\sigma^2} \right)^{-(\nu+1)/2}\nonumber
\end{IEEEeqnarray}
where $\nu>0$ is the degrees of freedom and $\sigma^2$ is the scale parameter. The challenge in this particular scenario lies in the fact that both its posterior and posterior predictive distributions no longer lend analytically tractable, necessitating the use of approximate inference methods. There are several ways to approximate this posterior: MCMC, Laplace’s method or variational inference \citep{jylanki2011robust}. Our experiments use the variational inference technique implemented in \texttt{GPflow} \citep{GPflow2017} to estimate the posterior, posterior predictive, and hyperparameter selection. \footnote{An official usage example of this implementation can be found here: \url{https://gpflow.github.io/GPflow/develop/notebooks/getting_started/classification_and_other_data_distributions.html\#Non-gaussian-regression}} In particular, the student-t \texttt{GPflow} implementation uses the variational approximation proposed by \citet{opper2009variational}. The main result of their work is that for a Gaussian process with a non-Gaussian likelihood, the optimal Gaussian approximation --- in terms of the Kullback-Leibler divergence --- is given by the expression: $q(\mathbf{f}) = \mathcal{N}(\mathbf{K}\boldsymbol{\alpha}, [\mathbf{K}^{-1} + \textrm{diag}(\boldsymbol{\lambda})]^{-1})$, where $\mathbf{K}$ is the kernel of the GP and $\boldsymbol{\alpha}$ and $\boldsymbol{\lambda}$ are the variational parameters

\paragraph{m-GP} 
This model explicitly considers the generation process of outliers as a uniform distribution 
over a bounded region--that covers the output $y$.
Then, each observation is associated with a latent variable $z\in\{0,1\}$, where $z=0$ indicates that the observation is generated by outlier distribution and $z = 1$ inlier. Therefore, the observation model is
\begin{IEEEeqnarray}{rCl}
   p(y,z| f(x), \gamma,\sigma^2)&=&\Big((1-\gamma)\frac{1}{a}\Big)^{1-z}\big(\gamma\mathcal{N}(y;f(x),\sigma^2)\big)^{z}\nonumber,
\end{IEEEeqnarray}
where $\gamma = p(z=1)$, and $a>0$ denotes the volume of the outlier region. The variable $\gamma$ controls the probability of occurrence of two models, and a Beta prior is assumed.
Like t-GP, its posterior and posterior predictive distributions no longer lend analytically tractable; thus, approximation is needed. We follow the official implementation provided in the paper. \footnote{\url{https://github.com/YifanLu2000/Robust-Scalable-GPR}}

\subsection{Benchmarking}\label{app:experiment_benchmarking}

\subsubsection{Description of the Datasets}

\paragraph{Synthetic} The dataset consists of $n=300$ samples from a GP with zero mean and squared exponential kernel (length scale and variance equal to 1). Then, we added Gaussian noise $\vepsilon\sim\mathcal{N}(0,0.3)$ to the observations. 
For the experiments, we selected a GP prior with mean function $m(x) = \frac{1}{n}\sum_{i=1}^n y_i$, and squared exponential kernel as covariance function.

\paragraph{Boston}The dataset consists of $n=506$ observations, each representing a suburban or town area in Boston. It encompasses $d=13$ features containing data like the average number of rooms in dwellings, pupil-teacher ratios, and per capita crime rates. We try to predict the median price of homes residents own (excluding rented properties).  The dataset can be found at \url{https://www.cs.toronto.edu/~delve/data/boston/bostonDetail.html}. We selected a GP prior with mean function $m(x) = \frac{1}{n}\sum_{i=1}^n y_i$, and squared exponential kernel as covariance function. 

\paragraph{Energy} The dataset describes the energy efficiency of buildings by correlating their heating and cooling load requirements with various building parameters. It consists of $n=768$ data samples, each characterised by $d=8$ distinct features, with the ultimate goal of predicting a single continuous response variable found in the last column. The dataset can be found at \url{https://archive.ics.uci.edu/dataset/242/energy+efficiency}. We selected a GP prior with mean function $m(x) = \frac{1}{n}\sum_{i=1}^n y_i$, and squared exponential kernel as covariance function. 

\paragraph{Yacht} The dataset's main focus is on predicting the residuary resistance of sailing yachts during their initial design phase, a critical aspect in evaluating a vessel's performance and estimating the essential propulsive power required. This prediction relies on $d=6$ primary input parameters, which include the fundamental hull dimensions and boat velocity. The dataset contains $n=308$ observations. The dataset can be found at  \url{https://archive.ics.uci.edu/dataset/243/yacht+hydrodynamics}. We selected a GP prior with mean function $m(x) = \frac{1}{n}\sum_{i=1}^n y_i$, and squared exponential kernel as covariance function.

\subsubsection{Description of the Outlier Generation Process}
\label{app:outliers_generation}
While contamination can occur in both the covariate $x$ and the observation $y$, our work is focused on cases where contamination exclusively affects the observations $y$. We now describe the three outlier generation processes presented in this paper. 

\paragraph{Uniform } In this setting, our initial step involves uniformly selecting a specified proportion of the dataset that will be contaminated. To uniformly contaminate the selection, we did a random 50-50 split of this subset: half of the selected subset is contaminated by adding $z\sim U(3\sigma, 9\sigma)$, while the other half is contaminated by subtracting $z\sim U(3\sigma, 9\sigma)$ where $\sigma$ is the standard deviation of the original observations, and $U$ denotes the uniform distribution. 

\paragraph{Asymmetric}
Much like the uniform outliers, we randomly select a subset of data points that we will contaminate. The key distinction lies in the fact that we do not split the selected subset; instead, we contaminate the entire subset by subtracting $z\sim U(3\sigma, 9\sigma)$ where $\sigma$ is the standard deviations of the original observations.

\paragraph{Focused} 
In this outlier generation process, we randomly select and remove a subset of data points, which will be replaced by outliers. 
For these outliers, we deterministically choose their values in $\mathcal{X}$. 
To do so, we calculate the median value for each input data dimension $j$. However, we do not place the outliers at this median position directly. 
Instead, we replace the removed input values by $(m_1 + \delta_1, m_2 + \delta_2 \dots, m_d + \delta_d)^{\top}$, where $m_j$ is the median in the $j$-th input data dimension, and $\delta_j = \alpha_j u$, where $\alpha_j$ is the median absolute deviation of the $j$-th data dimension times 0.1, and $u\sim U(0,1)$.
Simultaneously, the outlier values on $\mathcal{Y}$ are obtained by subtracting three times the standard deviation of the median of the observations $M_y$. To not have the same value for every outlier position, we also add a small perturbation $\delta_y = \alpha_y u$, where $\alpha_y$ is the median absolute deviation of $\y$ times 0.1, and $u\sim U(0,1)$.

\subsubsection{Additional Results}
We ran all benchmark experiments, choosing $c$ via leave-one-out (`c-LOO'), and compared it with the proposed way to choose $c$ (`c-$Q_n$').
Overall, the performance is slightly worse for c-LOO---likely because maximising the predictive posterior for extreme observations tends to match/fit these outliers by \textit{increasing} $c$, leading to a less robust method. 
Illustrative results can be seen in \cref{app:table_c_loo}.

\begin{table}[h]
\caption{Average test set mean absolute error and standard deviation (in brackets) for 50 train--test splits.}
\label{app:table_c_loo}
\centering
\begin{tabular}{@{}llll@{}}
\toprule
          & \multicolumn{1}{c}{\textcolor{green_plot}{\textbf{GP}}}       & \textcolor{blue_plot}{\textbf{$c$-LOO}}      & \textcolor{NavyBlue}{\textbf{$c$-${Q}_n$}}   \\ \midrule
          \multicolumn{4}{c}{{{\footnotesize No Outliers}}} 
          \vspace{0.5mm} \\ 
Synthetic & \textbf{0.09 (0.00)}  & \textbf{0.09 (0.00)} & \textbf{0.09 (0.00)} \\
Boston    & \textbf{0.19 (0.01)}  & \textbf{0.19 (0.01)} & \textbf{0.19 (0.01)}  \\
Energy    & 0.03 (0.00)  & \textbf{0.02 (0.00)} & \textbf{0.02 (0.00)}  \\
Yacht     & 0.02 (0.01)  & 0.02 (0.01) & 0.02 (0.01) \\ \midrule
\multicolumn{4}{c}{{{\footnotesize Focused Outliers}}}  \vspace{0.5mm}                          \\
Synthetic & 0.19 (0.00)  & 0.17 (0.00) & \textbf{0.15 (0.00)}\\
Boston    & 0.23 (0.06)  & \textbf{0.22 (0.03)} & \textbf{0.22 (0.01)}\\
Energy    & 0.03 (0.04)  & \textbf{0.02 (0.00)} & \textbf{0.02 (0.00)}\\
Yacht     & 0.26 (0.15)  & 0.11 (0.13) & \textbf{0.10 (0.14)}\\ \midrule
\multicolumn{4}{c}{{{\footnotesize Asymmetric Outliers}}} 
\vspace{0.5mm} \\
Synthetic & 1.14 (0.00)  & 1.08 (0.00) & \textbf{0.63 (0.00)} \\
Boston    & 0.63 (0.02)  & 0.61 (0.01) & \textbf{0.49 (0.00)} \\
Energy    & 0.54 (0.02)  & \textbf{0.37 (0.10)} & 0.44 (0.04) \\
Yacht     & 0.54 (0.06)  & 0.47 (0.03) & \textbf{0.35 (0.02)} \\
\bottomrule
\end{tabular}
\end{table}

\subsection{Sparse Variational Gaussian Processes}\label{app:SVGP}
 This section presents a numerical comparison between SVGP and RCSVGP on the four benchmark datasets presented in \cref{app:experiment_benchmarking}, with three outliers regimes: no outliers, focused outliers, asymmetric outliers. For both methods, we selected a GP prior with mean function $m(x) = \frac{1}{n}\sum_{i=1}^n y_i$, and squared exponential kernel as covariance function. We consider a fixed noise variance $\sigma^2 = 0.01$. \cref{tab:benchmark_sparse} shows that RCSVGP outperforms SVGP even in cases without outliers.
\begin{table}[h]
\caption{Average test set mean absolute error and standard deviation (in brackets) for 10 train--test splits.}
\label{tab:benchmark_sparse}
\centering
\begin{tabular}{@{}lllllll@{}}
\toprule
          & \multicolumn{1}{c}{\textcolor{green_plot}{\textbf{SVGP}}}       & \textcolor{blue_plot}{\textbf{SVRCGP}}      & \multicolumn{1}{c}{\textcolor{green_plot}{\textbf{SVGP}}}   & \textcolor{blue_plot}{\textbf{SVRCGP}} & \multicolumn{1}{c}{\textcolor{green_plot}{\textbf{SVGP}}}   & \textcolor{blue_plot}{\textbf{SVRCGP}} \\ \midrule
          \multicolumn{3}{c}{{{\footnotesize No Outliers}}} & \multicolumn{2 }{c}{{{\footnotesize Focused Outliers}}} & \multicolumn{2}{c}{{{\footnotesize Asymmetric Outliers}}} 
          \vspace{0.5mm} \\ 
Synthetic & \textbf{0.08 (0.00)}  & \textbf{0.08 (0.00)} & 0.19 (0.00)  & \textbf{0.13 (0.00)} & 1.32 (0.00)  & \textbf{1.02 (0.00)} \\
Boston    & 0.65 (0.01)  & \textbf{0.56 (0.18)}& 0.64 (0.02)  & \textbf{0.55 (0.16)}& 0.65 (0.04)  & \textbf{0.64 (0.01)}\\
Energy    & \textbf{0.02 (0.00)}  & 0.03 (0.00)& 0.05 (0.05)  & \textbf{0.05 (0.00}& 0.64 (0.05)  & \textbf{0.09 (0.16)}\\
Yacht     & \textbf{0.01 (0.00)}  & 0.03 (0.00)& 0.29 (0.00)  & \textbf{0.13 (0.02)}& 0.65 (0.04)  & \textbf{0.60 (0.02)}\\ 
\bottomrule
\end{tabular}
\end{table}

\subsection{Bayesian Optimisation}
\label{app:BO}
In the Bayesian optimisation experiment section, we compared RCGPs to GPs and t-GPs on two classical functions:
the Six-Hump Camel function and the Branin function. Here, we state the functions explicitly:
\begin{IEEEeqnarray}{rCll}
   \textbf{Six-Hump Camel:}\quad g_1(x,x')&=&\left(4-2.1x^2 +\frac{x^4}{3}\right)x^2  +  xx'  +  (4x'^2 - 4)x'^2 \quad& x \in (-2, 2), x' \in (-1, 1)\nonumber\\
   \textbf{Branin:}\quad g_2(x,x')&=&\left(x' - \frac{5.1}{4\pi^2}\: x^2 +\frac{5}{\pi} x -6\right)^2  +  10\left(1-\frac{1}{8\pi}\right) \cos(x)  +  10 \quad& x \in (-5, 10), x' \in (1, 15)\nonumber\\
   \textbf{McCormick:}\quad g_3(x,x')&=&\sin(x+x')+(x-x')+(-1.5x+2.5x'+1) \quad& x \in (-1.5, 4), x' \in (-3, 4)\nonumber\\
   \textbf{Rosenbrock:}\quad g_4(x,x')&=&100(x'-x^2)^2+(x^2-1)^2  \quad& x \in (-5, 10), x' \in (-5, 1)\nonumber
\end{IEEEeqnarray}
with global minima $g_1(x_{\star},x_{\star}^{\prime}) = -1.0316$, $g_2(x_{\star},x_{\star}^{\prime}) = 0.3979$, $g_3(x_{\star},x_{\star}^{\prime}) = -1.9133$ and $g_4(x_{\star},x_{\star}^{\prime}) = 0$.

In order to contaminate with outliers, we consider the scenario where, for each function evaluation, there is a 20\% chance of being
contaminated by an asymmetric outlier generated as \cref{app:outliers_generation}. In order to not change the global minimum, we consider the case where the outlier is bigger than the actual observation, i.e. asymmetric with adding a value.

We selected a GP prior with zero mean function and squared exponential kernel for the experiments for the three models. \cref{fig:bo_outliers_poi}  shows the results using PI as the acquisition function. In terms of cumulative regret, RCGPs outperform GPs. While t-GPs can match this, they take orders of magnitude longer to run. 
\begin{figure}[h]
    \centering
    \includegraphics{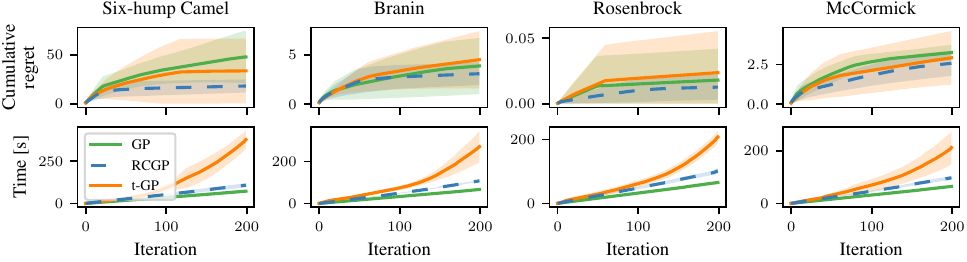}
    \caption{Mean cumulative regret (top) and clock time (bottom) for BO with 
    GP (\textcolor{green_plot}{green}), RCGP (\textcolor{blue_plot}{blue}) and t-GP (\textcolor{orange}{orange}) with 20\% asymmetric outliers and PI acquisition function over 10 realisations.}
    \label{fig:bo_outliers_poi}
\end{figure}
\newpage
\cref{fig:bo_no_outliers}  shows the results without outliers. In terms of cumulative regret, RCGPs match or sometimes outperform GPs. While t-GPs can match this, they take orders of magnitude longer to run. It is notable that even without outliers.
\begin{figure}[h]
    \centering
    \includegraphics{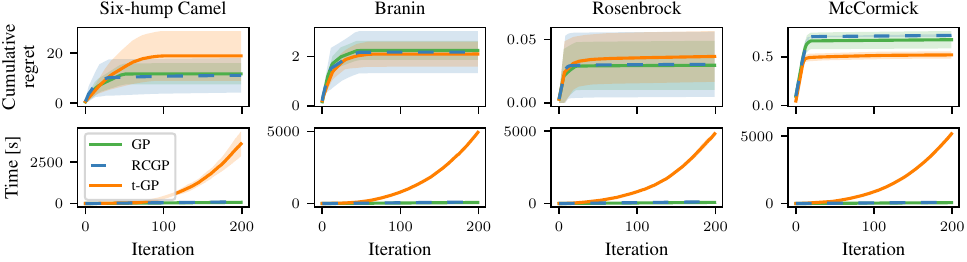}
    \caption{Mean cumulative regret (top) and clock time (bottom) for BO with 
    GP (\textcolor{green_plot}{green}), RCGP (\textcolor{blue_plot}{blue}) and t-GP (\textcolor{orange}{orange}) without outliers over 10 realisations.
    }
    \label{fig:bo_no_outliers}
\end{figure}

\end{document}